%% file: techreport_str_adm_alg.tex
\documentclass[11pt,twoside,a4paper]{article}

\usepackage{a4wide}
\usepackage{amsthm}
\usepackage{amssymb}
\usepackage{amsmath}
\usepackage[utf8]{inputenc}
\usepackage{xspace}
\usepackage{algorithm}
\usepackage{algpseudocode}
\usepackage{graphicx}

\newtheorem{definition}{Definition}
\newtheorem{theorem}{Theorem}
\newtheorem{lemma}[theorem]{Lemma}

\newcommand{\AF}{\mathit{AF}}
\newcommand{\Arguments}{\mathit{Ar}}
\newcommand{\attack}{\mathit{att}}
\newcommand{\Args}{\mathit{Args}}
\newcommand{\inn}{\mathtt{in}\xspace}
\newcommand{\out}{\mathtt{out}\xspace}
\newcommand{\undec}{\mathtt{undec}\xspace}
\newcommand{\Lab}{\mathcal{L}ab}
\newcommand{\MM}{\mathcal{MM}}
\newcommand{\MMLab}{\mathcal{MM}_{\Lab}}
\newcommand{\MMLabI}{{\mathcal{MM}_{\Lab}}_I}
\newcommand{\MMLabO}{{\mathcal{MM}_{\Lab}}_O}
\newcommand{\MMLabThree}{{\mathcal{MM}_{\Lab}}_3}
\newcommand{\ArgsToLab}{\mathtt{Args2Lab}}
\newcommand{\LabToArgs}{\mathtt{Lab2Args}}

\begin{document}

\title{Strong Admissibility, a Tractable Algorithmic Approach\\(proofs)}
\author{Martin Caminada\\Cardiff University
	\and
	Sri Harikrishnan\\Cardiff University}
\maketitle

\begin{abstract}
Much like admissibility is the key concept underlying preferred semantics,
strong admissibility is the key concept underlying grounded semantics, as
membership of a strongly admissible set is sufficient to show membership
of the grounded extension. As such, strongly admissible sets and labellings
can be used as an explanation of membership of the grounded extension, as
is for instance done in some of the proof procedures for grounded semantics.
In the current paper, we present two polynomial algorithms for constructing
relatively small strongly admissible labellings, with associated min-max
numberings, for a particular argument. These labellings can be used as 
relatively small explanations for the argument's membership of the grounded 
extension. Although our algorithms are not guaranteed to yield an absolute
minimal strongly admissible labelling for the argument (as doing do would
have implied an exponential complexity), our best performing algorithm 
yields results that are only marginally bigger. Moreover, the runtime of 
this algorithm is an order of magnitude smaller than that of the existing 
approach for computing an absolute minimal strongly admissible labelling 
for a particular argument. As such, we believe that our algorithms can be 
of practical value in situations where the aim is to construct a minimal 
or near-minimal strongly admissible labelling in a time-efficient way.
\end{abstract}

\input{sec01-intro}
\input{sec02-preliminaries}
\input{sec03a-algorithm}
\input{sec03b-algorithm}
\input{sec03c-algorithm}
\input{sec04-empirical-results}

\input{sec05-discussion}
\input{sec06-epilogue}

\bibliographystyle{plain}
\bibliography{../../ASPICbibliography}

\end{document}

%% file: sec01-intro.tex
\section{Introduction} \label{sec-intro}

In formal argumentation, one would sometimes like to show that a particular
argument is (credulously) accepted according to a particular argumentation
semantics, without having to construct the entire extension the argument is
contained in. For instance, to show
that an argument is in a preferred extension, it is not necessary to construct
the entire preferred extension. Instead, it is sufficient to construct a set 
of arguments that is \emph{admissible}. Similarly, to show that an argument
is in the grounded extension, it is not necessary to construct the entire
grounded extension. Instead, it is sufficient to construct a set of arguments
that is \emph{strongly admissible}.

The concept of strong admissibility was introduced by Baroni and Giacomin
\cite{BG07a} as one of the properties to describe and categorise argumentation
semantics. It was subsequently studied by Caminada and Dunne \cite{Cam14a,CD19a}
who further developed strong admissibility in both its set and labelling form.
In particular, the strongly admissible sets (resp. labellings) were found to 
form a latice with the empty set (resp. the all-$\undec$ labeling) as its 
bottom element and the grounded extension (resp. the grounded labelling) as
its top element \cite{Cam14a,CD19a}.

As a strongly admissible set (labelling) can be used to explain that a
particular argument is in the grounded extension (for instance, by using
the discussion game of \cite{Cam15a}) a relevant question is whether one
can identify an expanation that is \emph{minimal}. That is, given an
argument $A$ that is in the grounded extension, how can one obtain:\\
(1) a strongly admissible set that contains $A$, of which the number of
arguments is minimal among all strongly admissible sets containing $A$, and\\
(2) a strongly admissible labelling that labels $A$ $\inn$, of which the
number of $\inn$ and $\out$ labelled arguments (its \emph{size}, cf. 
\cite{CD20a}) is minimal among all strongly admissible labelings that label 
$A$ $\inn$. 

It has been found that the verification problem of (1) is NP-complete 
\cite{DW20} whereas the the verification problem of (2) is co-NP-complete 
\cite{CD20a}. Moreover, it has also been observed that even computing a 
c-approximation for the minimum size of a strongly admissible set for a 
given argument is NP-hard for every $c \geq 1$.
This is in sharp contrast with the complexity of the general verification
problem of strong admissibility (i.e. verifying whether a set/labelling
is strongly admissible, without the constraint that it also has to be minimal)
which has been found to be polynomial \cite{CD19a}.

The complexity results related to minimal strong admissibility pose a problem
when the aim is to provide the user with a relatively small explanation of
why a particular argument is in the grounded extension. For this, one can
either apply an algorithmic approach that yields an absolute minimal
explanation, but has an exponential runtime, or one can apply an algorithmic
approach that has a less than exponential runtime, but does not come with any
formal guarantees of how close the outcome is to an absolute minimal explanation
\cite{DW20}. The former approach is taken in \cite{DW20}. The latter approach
is taken in our current paper.

In the absence of a dedicated algorithm for strong admissibility, one may be 
tempted to simply apply an algorithm for computing the grounded extension
or labelling instead (such as \cite{MC09,NAD21}) if the aim is to do the
computation in polynomial time. Still, from the perspective of minimality,
this would yield the absolute worst outcome, as the grounded extension
(labeling) is the maximal strongly admissible set (labelling). In the
current paper we therefore introduce an alternative algorithm which, like
the grounded semantics algoritms, runs in polynomial time but tends to produce 
a strongly admissible set (resp. labelling) that is that is significantly
smaller than the grounded extension (resp. labelling). As the complexity results
from \cite{DW20} prevent us from giving any theory-based guarantees regarding
how close the outcome of the algorithm is to an absolute minimal strongly
admissible set, we will instead assess the performance of the algorithm using
a wide range of benchmark examples.

The remaining part of the current paper is structured as follows.
First, in Section \ref{sec-preliminaries} we give a brief overview
of the formal concepts used in the current paper, including that of
a strongly admissible set and a strongly admissible labelling. 
In Section \ref{sec-algorithm} we then proceed to provide the proposed 
algorithm, including the associated proofs of correctness. Then, in Section 
\ref{sec-empirical-results} we assess the performance of our approach,
and compare it with the results yielded by the approach in \cite{DW20}
both in terms of outcome and runtime. 
We round off with a discussion of our findings in Section \ref{sec-discussion}.

%% file: sec02-preliminaries.tex
\section{Preliminaries} \label{sec-preliminaries}

In the current section, we briefly restate some of the basic concepts in
formal argumentation theory, including strong admissibility. For current
purposes, we restrict ourselves to finite argumentation frameworks.

\begin{definition} \label{def-AF}
An \emph{argumentation framework} is a pair $(\Arguments, \attack)$ where
$\Arguments$ is a finite set of entities, called arguments, whose internal
structure can be left unspecified, and $\attack$ is a binary relation on
$\Arguments$. For any $x,y \in \Arguments$ we say that $x$ \emph{attacks} 
$y$ iff $(x,y) \in \attack$.
\end{definition}


As for notation, we use lower case letters at the end of the alphabet
(such as $x$, $y$ and $z$) to denote variables containing arguments,
upper case letters at the end of the alphabet (such as $X$, $Y$ and $Z$)
to denote program variables containing arguments, and upper case letters
at the start of the alphabet (such as $A$, $B$ and $C$) to denote concrete
instances of arguments.

When it comes to defining argumentation semantics, one can distinguish
the \emph{extension approach} and the \emph{labelling approach}
\cite{BCG17}. We start with the extensions approach.

\begin{definition} \label{def-cf-defends}
Let $(\Arguments, \attack)$ be an argumentation framework, $x \in \Arguments$
and $\Args \subseteq \Arguments$. We define $x^+$ as $\{ y \in \Arguments \mid
x$ attacks $y \}$, $x^-$ as $\{ y \in \Arguments \mid y$ attacks $x \}$,
$\Args^+$ as $\bigcup \{ x^+ \mid x \in \Args \}$, and $\Args^-$ as 
$\bigcup \{ x^- \mid x \in \Args \}$. $\Args$ is said to be \emph{conflict-free}
iff $\Args \cap \Args^+ = \emptyset$. $\Args$ is said to \emph{defend} $x$ iff 
$x^- \subseteq \Args^+$. The characteristic function $F: 2^{\Arguments} 
\rightarrow 2^{\Arguments}$ is defined as $F(\Args) = \{ x \mid \Args$ defends 
$x \}$.
\end{definition}

\begin{definition} \label{def-ext-semantics}
Let $(\Arguments, \attack)$ be an argumentation framework. $\Args 
\subseteq \Arguments$ is
\begin{itemize}
  \item an admissible set iff 
	$\Args$ is conflict-free and $\Args \subseteq F(\Args)$
  \item a complete extension iff 
	$\Args$ is conflict-free and $\Args = F(\Args)$
  \item a grounded extension iff $\Args$ is 
	the smallest (w.r.t. $\subseteq$) complete extension
  \item a preferred extension iff $\Args$ is 
	a maximal (w.r.t. $\subseteq$) complete extension
\end{itemize}
\end{definition}

As mentioned in the introduction, the concept of strong admissibility was 
originally introduced by Baroni and Giacomin \cite{BG07a}. For current purposes
we will apply the equivalent definition of Caminada \cite{Cam14a,CD19a}.

\begin{definition} \label{def-strong-adm}
Let $(\Arguments, \attack)$ be an argumentation framework. 
$\Args \subseteq \Arguments$ is \emph{strongly admissible} iff 
every $x \in \Args$ is defended by some $\Args' \subseteq \Args \setminus
\{ x \}$ which in its turn is again strongly admissible.
\end{definition}

\begin{figure}[thb]
\centering
\includegraphics[scale=1.0]{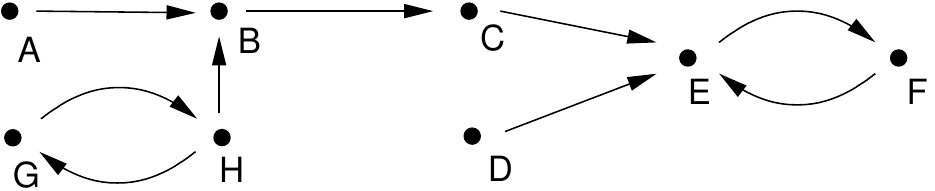}
\caption{An example of an argumentation framework.\label{fig-example-AF}}
\end{figure}

As an example (taken from \cite{CD19a}), in the argumentation framework
of Figure \ref{fig-example-AF} the strongly admissible sets are $\emptyset$, 
$\{A\}$, $\{A,C\}$, $\{A,C,F\}$, $\{D\}$, $\{A,D\}$, $\{A,C,D\}$, $\{D,F\}$, 
$\{A,D,F\}$ and $\{A,C,D,F\}$, the latter also being the grounded extension. 
The set $\{A,C,F\}$ is strongly admissible as $A$ is defended by $\emptyset$,
$C$ is defended by $\{A\}$ and $F$ is defended by $\{A,C\}$, each of which
is a strongly admissible subset of $\{A,C,F\}$ not containing the argument 
it defends. Please notice that although the set $\{A,F\}$ defends argument 
$C$ in $\{A,C,F\}$, it is in its turn not strongly admissible (unlike $\{A\}$).
Hence the requirement in Definition \ref{def-strong-adm} for $\Args'$ to be 
a \emph{subset} of $\Args \setminus \{A\}$. We also observe that although 
$\{C, H\}$ is an admissible set, it is not a \emph{strongly} admissible set, 
since no subset of $\{C, H\} \setminus \{H\}$ defends $H$.

It can be shown that each strongly admissible set is conflict-free and
admissible \cite{CD19a}. The strongly admissible sets form a lattice (w.r.t. 
$\subseteq$), of which the empty set is the bottom element and the grounded
extension is the top element \cite{CD19a}.

The above definitions essentially follow the extension based approach as
described in \cite{Dun95}. It is also possible to define the key argumentation 
concepts in terms of argument labellings \cite{Cam06d,CG09}.

\begin{definition}\label{def-lab-semantics-i}
Let $(\Arguments, \attack)$ be an argumentation framework. An \emph{argument
labelling} is a function $\Lab: \Arguments \rightarrow \{ \inn, \out,
\undec \}$. An argument labelling is called an \emph{admissible labelling}
iff for each $x \in \Arguments$ it holds that:
\begin{itemize}
  \item if $\Lab(x) = \inn$ then for each $y$ that attacks $x$
	it holds that $\Lab(y) = \out$
  \item if $\Lab(x) = \out$ then there exists a $y$ that attacks $x$
	such that $\Lab(y) = \inn$
\end{itemize}
$\Lab$ is called a \emph{complete labelling} iff it is an admissible labelling
and for each $x \in \Arguments$ it also holds that:
\begin{itemize}
  \item if $\Lab(x) = \undec$ then 
	there is a $y$ that attacks $x$ such that $\Lab(y) = \undec$, and
	for each $y$ that attacks $x$ such that $\Lab(y) \neq \undec$ it
	holds that $\Lab(y) = \out$
\end{itemize}
\end{definition}

As a labelling is essentially a function, we sometimes write it as a set of
pairs. Also, if $\Lab$ is a labelling, we write $\inn(\Lab)$ for $\{ x \in
\Arguments \mid \Lab(x) = \inn \}$, $\out(\Lab)$ for $\{ x \in \Arguments \mid
\Lab(x) = \out \}$ and $\undec(\Lab)$ for $\{ x \in \Arguments \mid \Lab(x)
= \undec \}$. As a labelling is also a partition of the arguments into 
sets of $\inn$-labelled arguments, $\out$-labelled arguments and 
$\undec$-labelled arguments, we sometimes write it as a triplet $(\inn(\Lab), 
\out(\Lab), \undec(\Lab))$.

\begin{definition}[\cite{CP11}]\label{def-lab-operators}
Let $\Lab$ and $\Lab'$ be argument labellings 
of argumentation framework $(\Arguments, \attack)$. 
We say that $\Lab \sqsubseteq \Lab'$ iff 
$\inn(\Lab) \subseteq \inn(\Lab')$ and $\out(\Lab) \subseteq \out(\Lab')$. 
\end{definition}


\begin{definition}\label{def-lab-semantics-ii}
Let $\Lab$ be a complete labelling of argumentation framework
$(\Arguments, \attack)$. $\Lab$ is said to be
\begin{itemize}
  \item the grounded labelling iff $\Lab$ is the (unique) smallest (w.r.t.
	$\sqsubseteq$) complete labelling
  \item a preferred labelling iff $\Lab$ is a maximal (w.r.t. $\sqsubseteq$)
	complete labelling
\end{itemize}
\end{definition}

We refer to the \emph{size} of a labelling $\Lab$ as $|\inn(\Lab) \cup
\out(\Lab)|$. We observe that if $\Lab \sqsubseteq \Lab'$ then the size
of $\Lab$ is smaller or equal to the size of $\Lab'$, but not necessarily
vice versa. In the remainder of the current paper, we use the terms smaller,
bigger, minimal and maximal in relation to the size of the respective
labellings, unless stated otherwise.

The next step is to define a strongly admissible labelling. In order to do
so, we need the concept of a min-max numbering \cite{CD19a}.

\begin{definition} \label{def-min-max-numbering}
Let $\Lab$ be an admissible labelling of argumentation framework $(\Arguments,
\attack)$. A \emph{min-max numbering} is a total function 
$\MM_{\Lab}: \inn(\Lab) \cup \out(\Lab) \rightarrow \mathbb{N} \cup \{\infty\}$
such that for each $x \in \inn(\Lab) \cup \out(\Lab)$ it holds that:
\begin{itemize}
  \item if $\Lab(x) = \inn$ then 
	$\MM_{\Lab}(x) = max(\{\MM_{\Lab}(y) \mid y$ attacks $x$ and 
	$\Lab(y) = \out \}) + 1$
	(with $max(\emptyset)$ defined as $0$)
  \item if $\Lab(x) = \out$ then
	$\MM_{\Lab}(x) = min(\{\MM_{\Lab}(y) \mid y$ attacks $x$ and
	$\Lab(y) = \inn \}) + 1$
	(with $min(\emptyset)$ defined as $\infty$)
\end{itemize}
\end{definition}

It has been proved that every admissible labelling has a unique min-max
numbering \cite{CD19a}. A strongly admissible labelling can then be defined
as follows \cite{CD19a}.

\begin{definition} \label{def-strongly-adm-lab}
A \emph{strongly admissible labelling} is an admissible labelling whose
min-max numbering yields natural numbers only (so no argument is numbered
$\infty$).
\end{definition}

As an example (taken from \cite{CD19a}), consider again the argumentation 
framework of Figure \ref{fig-example-AF}. 
Here, the admissible labelling $\Lab_1 = ( \{A,C,F,G\}, \{B,E,H\}, \{D\} )$ 
has min-max numbering $\{ (A:1), (B:2), (C:3), (E:4), (F:5), (G:\infty), 
(H:\infty) \}$, which means that it is not strongly admissible. The admissible 
labelling $\Lab_2 = ( \{A,C,D,F\}, \{B,E\}, \{G,H\} )$ has min-max numbering 
$\{ (A:1), (B:2), (C:3), (D:1), (E:2), (F:3) \}$, which means that it
is strongly admissible.

It has been shown that the strongly admissible labellings form a lattice
(w.r.t. $\sqsubseteq$), of which the all-$\undec$ labelling is the bottom 
element and the grounded labelling is the top element \cite{CD19a}.


The relationship between extensions and labellings has been well-studied
\cite{Cam06a,CG09}. A common way to relate extensions to labellings is
through the functions $\ArgsToLab$ and $\LabToArgs$. These translate a
conflict-free set of arguments to an argument labelling, and an argument 
labelling to a set of arguments, respectively. More specifically, given
an argumentation framework $(\Arguments, \attack)$, and an associated
conflict-free set of arguments $\Args$ and a labelling $\Lab$, 
$\ArgsToLab(\Args)$ is defined as $(\Args, \Args^+, \Arguments \setminus 
(\Args \cup \Args^+))$ and $\LabToArgs(\Lab)$ is defined as $\inn(\Lab)$.
It has been proven \cite{CG09} that if $\Args$ is an admissible set 
(resp. a complete, grounded or preferred extension) then $\ArgsToLab(\Args)$ 
is an admissible labelling (resp. a complete, grounded or preferred labelling),
and that if $\Lab$ is an admissible labelling (resp. a complete, grounded or 
preferred labelling) then $\LabToArgs(\Lab)$ is an admissible set (resp.
a complete, grounded or preferred extension). It has also been proven 
\cite{CD19a} that if $\Args$ is a strongly admissible set then 
$\ArgsToLab(\Args)$ is a strongly admissible labelling, and that if $\Lab$ 
is a strongly admissible labelling then $\LabToArgs(\Lab)$ is a strongly
admissible set.


%% file: sec03a-algorithm.tex
\section{The Algorithms} \label{sec-algorithm}

In the current section, we present an algorithmic approach for computing a
relatively small\footnote{Small with respect to the size of the labelling.}
strongly admissible labelling. For this, we provide three different algorithms.
The first algorithm (Algorithm \ref{alg-construct}) basically constructs a 
strongly admissible labelling bottom-up, starting with the arguments that have 
no attackers and continuing until the main argument (the argument for which 
one want to show membership of a strongly admissible set) is labelled $\inn$. 
The second algorithm (Algorithm \ref{alg-prune}) then takes the
output of the first algorithm and tries to prune it. That is, it tries to 
identify only those $\inn$ and $\out$ labelled arguments that are actually 
needed in the strongly admissible labelling. The third algorithm (Algorithm
\ref{alg-combine}) then combines Algorithm \ref{alg-construct} (which is used 
as the construction phase) and Algorithm \ref{alg-prune} (which is used as
the pruning phase).

\subsection{Algorithm \ref{alg-construct}} \label{subsec-construct}

The basic idea of Algorithm \ref{alg-construct} is to start constructing the 
grounded labelling bottom-up, until we reach the main argument (that is, until 
we reach the argument that we are trying to construct a strongly admissible 
labelling for; this argument should hence be labelled $\inn$). As such, the 
idea is to take an algorithm for computing the grounded labeling (e.g. 
\cite{MC09} or \cite{NAD21}) and modify it accordingly. We have chosen the 
algorithm of \cite{NAD21} for this purpose, as it has been proved to run 
faster than some of the alternatives (such as \cite{MC09}). We had to adjust 
this algorithm in two ways. First, as mentioned above, we want the algorithm 
to stop once it hits the main argument, instead of continuing to construct the 
entire grounded labelling. Second, we want it to compute not just the strongly
admissible labelling itself, but also its associated min-max numbering.

Obtaining the min-max numbering is important, as it can be used to show
that the obtained admissible labelling is indeed \emph{strongly} admissible,
through the absence of $\infty$ in its min-max numbering. Additionally, the
min-max numbering is also needed for some of the applications of strong
admissibility, in particular the Grounded Discussion Game \cite{Cam15a} where
the combination of a strongly admissible labelling and its associated
min-max numbering serves as a roadmap for obtaining a winning strategy.

Instead of first computing the strongly admissible labelling and then
proceeding to compute the min-max numbering, we want to compute both
the strongly admissible labelling and the min-max numbering in just a
single pass, in order to achieve the best performance.

\begin{algorithm}
\caption{Construct a strongly admissible labelling that labels $A$ $\inn$
and its associated min-max numbering.} \label{alg-construct}
\begin{algorithmic}[1]
\Statex \textbf{Input}: An argumentation framework $\AF=(\Arguments,\attack)$,
\Statex an argument $A \in \Arguments$ that is in the grounded extension 
	of $\AF$.
\Statex \textbf{Output}: A strongly admissible labelling $\Lab$ where $A \in \inn(\Lab)$,
\Statex the associated min-max numbering $\MMLab$.
\Statex

\State  // We start with the type definitions
\State  $\Lab: \Arguments \rightarrow \{\inn, \out, \undec\}$
\State  $\MMLab: \inn(\Lab) \cup \out(\Lab) \rightarrow \mathbb{N} \cup \{\infty\}$
\State  $\mathtt{undec\_pre}: \Arguments \rightarrow \mathbb{N}$
\State  $\mathtt{unproc\_in}: [X_1, ... X_n]$
	($X_i \in \Arguments$ for each $1 \leq i \leq n$)
	// list of arguments
\State
\State  // Next, we initialize and process the arguments that have no attackers
\State  $\mathtt{unproc\_in} \leftarrow [ ]$
        \For{each $X \in \Arguments$}
\State      $\Lab(X) \leftarrow \undec$
\State      $\mathtt{undec\_pre}(X) \leftarrow | X^- |$
            \If{$\mathtt{undec\_pre}(X) = 0$}
\State          add $X$ to the rear of $\mathtt{unproc\_in}$
\State          $\Lab(X) \leftarrow \inn$
\State          $\MMLab(X) \leftarrow 1$
\State          \textbf{if} $X = A$ \textbf{then} 
		    return $\Lab$ and $\MMLab$
            \EndIf
        \EndFor
\State
\State  // We proceed to process the arguments that do have attackers
        \While{$\mathtt{unproc\_in}$ is not empty}
\State      let $X$ be the argument at the front of $\mathtt{unproc\_in}$
\State      remove $X$ from $\mathtt{unproc\_in}$
            \For{each $Y \in X^+$ with $\Lab(Y) \neq \out$}
\State          $\Lab(Y) \leftarrow \out$
\State          $\MMLab(Y) \leftarrow \MMLab(X) + 1$
                \For{each $Z \in Y^+$ with $\Lab(Z) = \undec$}
\State              $\mathtt{undec\_pre}(Z) \leftarrow \mathtt{undec\_pre}(Z)-1$
                    \If{$\mathtt{undec\_pre}(Z) = 0$}
\State                  add $Z$ to the rear of $\mathtt{unproc\_in}$
\State                  $\Lab(Z) \leftarrow \inn$
\State                  $\MMLab(Z) \leftarrow \MMLab(Y) + 1$
\State                  \textbf{if} $Z = A$ \textbf{then}
			    return $\Lab$ and $\MMLab$
                    \EndIf
                \EndFor
            \EndFor
        \EndWhile
\State
\State // If we get here, A is not in the grounded extension,
\State // so we may want to print an error message
\end{algorithmic}
\end{algorithm}

To see how the algorithm works, consider again the argumentation
framework of Figure \ref{fig-example-AF}. Let $C$ be the main argument.
At the start of the first iteration of the while loop (line 21) it holds
that $\Lab = (\{A,D\}, \emptyset, \{B,C,E,F,G,H\})$, $\MMLab = \{(A:1), (D:1)\}$
and $\mathtt{unproc\_in} = [A,D]$. At the first iteration of the while loop,
the argument in front of $\mathtt{unproc\_in}$ ($A$) is selected (line 22).
This then means that $B$ gets labelled $\out$ and $C$ gets labelled $\inn$
Hence, the algorithm hits the main argument ($C$) at line 33 and terminates.
This yields a labelling $\Lab = (\{A,C,D\}, \{B\}, \{E,F,G,H\})$ and associated
min-max numbering $\MMLab = \{(A:1), (B:2), (C:3), (D:1)\}$.

We now proceed to prove some of the formal properties of the algorithm.
The first property to be proved is termination.

\begin{theorem} \label{th-a1-terminates}
Let $\AF = (\Arguments, \attack)$ be an argumentation framework 
and $A$ be an argument in the grounded extension of $\AF$. Let both $\AF$
and $A$ be given as input to Algorithm \ref{alg-construct}. It holds that the 
algorithm terminates.
\end{theorem}

\begin{proof}
As for the first loop (the for loop of lines 9-18) we observe that it 
terminates as the number of arguments in Ar is finite.

As for the second loop (the while loop of lines 21-37) we first observe that 
no argument can be added to \texttt{unproc\_in} more than once (that is, once 
an argument has been added to \texttt{unproc\_in}, it can never be added again).
This is because for an argument to be added, it has to be labelled $\undec$
(line 27) whereas after adding it, it will be labelled $\inn$ (line 31).
Moreover, once an argument is labelled $\inn$, it will stay labelled $\inn$ 
as there is nothing in the algorithm that will change it. Given that 
(1) there is only a finite number of arguments in $\Arguments$, 
(2) each argument can be added to \texttt{unproc\_in} at most once, and 
(3) each iteration of the while loop removes an argument from 
    \texttt{unproc\_in},
it follows that the loop has to terminate.
\end{proof}

Next, we need to show that the algorithm is correct. That is, we need to
show that the algorithm yields a strongly admissible labelling $\Lab$ that
labels $A$ $\inn$, together with its associated min-max numbering $\MMLab$.
In order to do so, we first need to state and prove a number of lemmas.
We start with showing that $\Lab$ is admissible in every stage of the 
algorithm.

\begin{lemma} \label{lemma-a1-admissible}
Let $\AF = (\Arguments, \attack)$ be an argumentation framework 
and $A$ be an argument in the grounded extension of $\AF$. Let both $\AF$ 
and $A$ be given as input to Algorithm \ref{alg-construct}. It holds that 
during any stage in the algorithm, $\Lab$ is an admissible labelling.
\end{lemma}

\begin{proof}
Consider the value of $\Lab$ at an arbitrary point during the execution
of Algorith 1. According to the definition of an admissible labelling
(Definition \ref{def-lab-semantics-i}) we need to prove two things, for an 
arbitrary argument
$x \in \Arguments$:
\begin{enumerate}
  \item if $\Lab(x) = \inn$ then for each $y$ that attacks $x$
	it holds that $\Lab(y) = \out$\\
	Suppose $\Lab(x) = \inn$. We distinguish two cases:
	\begin{enumerate}
	  \item $x$ was labelled $\inn$ at line 14. This implies that 
		$\mathtt{und\_pre}(x) = 0$ in line 12, which implies that 
		$x$ has no attackers. Therefore, trivially $\Lab(y) = \out$
		for each $y \in \Arguments$ that attacks $x$.
	  \item $x$ was labelled $\inn$ at line 31. This implies that 
		$\mathtt{undec\_pre}(x) = 0$ in line 29, which implies that 
		each attacker $y$ of $x$ has been relabelled to $\out$.
		To see that this is the case, let $n$ be the number of
		attackers of $x$ (that is, $n = | x^- |$). It follows that
		$\mathtt{undec\_pre}(x)$ is initially $n$ (line 11) and at
		least 1 (otherwise $x$ would have been labelled $\inn$ at line
		14 instead of at line 31). In order for $\mathtt{undec\_pre}(x)$
		to have fallen to 0 (line 29) it will need to have decremented
		(at line 28) $n$ times (as no other line changes the value of
		$\mathtt{undec\_pre}(x)$). Each time this happens at line 28,
		an attacker of $x$ that wasn't previously labelled $\out$
		(line 24) is labelled $\out$ (line 25). Therefore, by the
		time $\mathtt{undec\_pre}(x)$ became 0, it follows that
		\emph{all} attackers of $x$ have become labelled $\out$.
	\end{enumerate}
  \item if $\Lab(x) = \out$ then there exists a $y$ that attacks $x$
	such that $\Lab(y) = \inn$\\
	Suppose $\Lab(x) = \out$. This implies that $x$ was labelled $\out$
	at line 25, which implies that an attacker $y$ of $x$ was an element 
	of $\mathtt{unproc\_in}$. This means that at some point, argument $y$
	was added to $\mathtt{unproc\_in}$. This could have happened at 
	line 13 or 30. In both cases, it follows that (line 14 and 31) $y$
	is labelled $\inn$.
\end{enumerate}
\end{proof}

The next lemma presents an intermediary result that will be needed further
on in the proofs.

\begin{lemma} \label{lemma-a1-min-1}
Let $\AF = (\Arguments, \attack)$ be an argumentation framework and 
$A$ be an argument in the grounded extension of $\AF$. Let both $\AF$ and
$A$ be given as input to Algorithm \ref{alg-construct}. It holds that for 
each argument $x$ that is added to $\mathtt{unproc\_in}$, $\MMLab(x) \geq 1$
\end{lemma}

\begin{proof}
We prove this by induction over the number of arguments that are added
to $\mathtt{unproc\_in}$ during the execution of the while loop of lines
21-37.
\begin{description}
  \item[BASIS (n=0)]
	Suppose the while loop has not yet added any arguments to
	$\mathtt{unproc\_in}$. This means that any argument $x$ that
	was added to $\mathtt{unproc\_in}$ was added by the for loop
	(lines 9-18). This could only have been done at line 13.
	Line 15 then implies that $\MMLab(x) =1$ so trivially
	$\MMLab(x) \geq 1$.
  \item[STEP]
	Suppose that at a particular point, the while loop has added
	$n$ ($\geq 0$) arguments to $\mathtt{unproc\_in}$ and that for
	each argument $x$ that has been added to $\mathtt{unproc\_in}$
	(either by the while loop of lines 21-37 or by the for loop of
	lines 9-18) it holds that $\MMLab(x) \geq 1$. 
	We distinguish two cases:
	\begin{itemize}
	  \item $x$ was added to $\mathtt{unproc\_in}$ previously.
		From the induction hypothesis it follows that $\MMLab(x)
		\geq 1$ at the moment $x$ was added. As Algorithm 
		\ref{alg-construct} does not change any value of $\MMLab$ 
		once it is assigned, it follows that $\MMLab(x) \geq 1$ 
		still holds at the current point.
	  \item $x$ is the argument that is currently being added to
		$\mathtt{unproc\_in}$ (so $x = Z$ at line 30). This
		implies that $Z$ is labelled $\inn$ at line 31 and is
		numbered $\MMLab(Y)+1$ at line 32. Following line 26,
		it holds that $\MMLab(Y) = \MMLab(X)+1$, with $X$ being
		an $\inn$ labelled attacker of $Y$ that was added to
		$\mathtt{unproc\_in}$ previously. We can therefore
		apply the induction hypothesis and obtain that $\MMLab(X)
		\geq 1$, which together with the earlier observed facts
		that $\MMLab(Z) = \MMLab(Y)+1$ (line 32) and $\MMLab(Y) = 
		\MMLab(X)+1$ (line 26) implies that $\MMLab(Z) \geq 3$
		which trivially implies that $\MMLab(Z) \geq 1$. Hence
		(as $x = Z$) we obtain that $\MMLab(x) \geq 1$.
	\end{itemize}
\end{description}
\end{proof}

Algorithm \ref{alg-construct} (especially line 22 and line 30) implements a 
FIFO queue for the $\inn$ labelled arguments it processes. This is an important
difference with the algorithm of \cite{NAD21}, which uses a set for this
purpose. Using a set is fine if the aim is merely to compute a strongly
admissible labelling (as is the case for \cite{NAD21} where the aim is to
compute the grounded labelling). However, if the aim is also to compute the 
associated min-max numbering, having a set as the basic data structure could 
compromise the algorithm's correctness. 

As an example, consider again the argumentation framework of Figure 
\ref{fig-example-AF}. Let $F$ be the main argument. Now suppose that
$\mathtt{unproc\_in}$ is a set instead of a queue. In that case, at the
start of the first iteration of the while loop (line 21) it holds that
$\Lab = (\{A,D\}, \emptyset, \{B,C,E,F,G,H\})$, $\MMLab = \{(A:1), (D:1)\}$
and $\mathtt{unproc\_in} = \{A,D\}$. At the first iteration of the while loop,
an argument $X$ from $\mathtt{unproc\_in}$ is selected (line 22). As a set
has no order, it would be possible to select $A$ (so $X=A$). This then
means that $B$ gets labelled $\out$ and $C$ gets labelled $\inn$.
Hence, at the end of the first iteration of the while loop (and therefore
at the start of the second iteration of the while loop) it holds that
$\Lab = (\{A,C,D\}, \{B\}, \{E,F,G,H\})$, $\MMLab = \{(A:1), (B:2), (C:3),
(D:1)\}$ and $\mathtt{unproc\_inn} = \{C,D\}$. At the second iteration of
the while loop, suppose $C$ is the selected argument (so $X=C$). This means
that $E$ gets labelled $\out$ and $F$ gets labelled $\inn$. Hence, at the 
moment the algorithm hits the main argument ($F$, at line 33) and terminates,
it holds that $\Lab = (\{A,C,D,F\}, \{B,E\}, \{G,H\})$ and $\MMLab =
\{(A:1), (B:2), (C:3), (D:1), (E:4), (F:5)\}$. Unfortunately $\MMLab$ is
incorrect. This is because $\out$ labelled argument $E$ is numbered 4,
whereas its two $\inn$ labelled attackers $C$ and $D$ are numbered 3 and 1,
respectively, so the correct min-max number of $E$ should be 2 instead of 4,
which implies that the correct min-max number of $F$ should be 3 instead of 5.

One of the conditions of a min-max numbering is that the min-max number of
an $\out$ labelled argument should be the minimal value of its $\inn$ labelled
attackers, plus 1. This seems to require that the min-max number of the $\inn$
labelled attackers is already known, before assigning the min-max number of
the $\out$ labelled argument. At the very least, it would seem that the min-max
number of an $\out$ labelled argument would potentially need to be recomputed
each time the min-max number of one of its $\inn$ labelled attackers becomes 
known. Yet, Algorithm \ref{alg-construct} does none of this. It determines the
min-max number of an $\out$ labelled argument as soon as the min-max number of 
its first $\inn$ labelled attacker becomes known (line 26) without waiting for 
the min-max number of any other $\inn$ labelled attacker becoming available.
Yet, Algorithm \ref{alg-construct} still somehow manages to always yield the 
correct min-max numbering.

The key to understanding how Algorithm \ref{alg-construct} manages to always 
yield the correct min-max numbering is that the $\inn$ labelled arguments are 
processed in the order of their min-max numbers. That is, once an $\inn$ 
labelled attacker is identified, any subsequently identified $\inn$ labelled 
attacker will have a min-max number greater or equal to the first one and will 
therefore not change the minimal value (in the sense of Definition 
\ref{def-min-max-numbering}, first bullet point). This avoids having to
recalculate the min-max number of an $\out$ labelled attacker once more
of its $\inn$ labelled arguments become available, therefore speeding up
the algorithm.

To make sure that arguments are processed in the order of their min-max
numbers, we need to apply a FIFO queue instead of the set that was applied by
\cite{NAD21}. The following two lemmas (Lemma \ref{lemma-a1-sorted-adding}
and Lemma \ref{lemma-a1-sorted-removing}) state that the $\inn$ labelled
arguments are indeed added and removed to the queue in the order of their
min-max numbers. These properties are subsequently used to prove the 
correctness of the computed min-max numbering (Lemma \ref{lemma-a1-mm-correct}
and Theorem \ref{th-a1-correct}).

\begin{lemma} \label{lemma-a1-sorted-adding}
Let $\AF = (\Arguments, \attack)$ be an argumentation framework and 
$A$ be an argument in the grounded extension of $\AF$. Let both $\AF$ and 
$A$ be given as input to Algorithm \ref{alg-construct}. The order in which 
arguments are added to $\mathtt{unproc\_in}$ is non-descending w.r.t. $\MMLab$.
That is, if argument $x_1$ is added to $\mathtt{unproc\_in}$ before argument 
$x_2$ is added to $\mathtt{unproc\_in}$, then $\MMLab(x_1) \leq \MMLab(x_2)$.
\end{lemma}

\begin{proof}
We first observe that this property is satisfied just after finishing the
for loop of lines 9-18. This is because the for loop makes sure that for
each argument $x$, $\MMLab(x) = 1$ (line 15) so it is trivially satisfied 
that if $x_1$ is added before $x_2$, then $\MMLab(x_1) \leq \MMLab(x_2)$.
We proceed the proof by induction over the number of arguments added by
the while loop (lines 21-37).
\begin{description}
  \item[BASIS (n=0)]
	Suppose no arguments have yet been added to $\mathtt{unproc\_in}$
	by the while loop. In that case, all arguments that have been added 
	to $\mathtt{unproc\_in}$ were added by the for loop (lines 9-18) 
	for which we have observed that the property holds.
  \item[STEP (n+1)]
	Suppose the property holds after n arguments have been added to 
	$\mathtt{unproc\_in}$ by the while loop. We now show that if the 
	while loop adds another argument (n+1) to $\mathtt{unproc\_in}$, 
	the property still holds. In the while loop, only line 30 adds 
	an argument to $\mathtt{unproc\_in}$. Let $Z_{new}$ be the argument 
	(n+1) that is currently added and let $Z_{old}$ be an argument that 
	was previously added. We distinguish two cases:
	\begin{enumerate}
	  \item $Z_{old}$ has been added by the while loop (so at a previous 
		run of line 30). Let $Y_{new}$ be the $\out$ labelled attacker 
		of $Z_{new}$ at line 25 and $Y_{old}$ be the $\out$ labelled 
		attacker of $Z_{old}$ at line 25. Let $X_{new}$ be the $\inn$
		labelled attacker of $Y_{new}$ at line 22 and let $X_{old}$
		be the $\inn$ labelled attacker of $Y_{old}$ at line 22.
		It holds that either
		\begin{enumerate}
		  \item $X_{old}$ was added to $\mathtt{unproc\_in}$ before 
			$X_{new}$ (at a previous iteration of the while loop,
			in which case it follows from our induction hypothesis 
			that $\MMLab(X_{old}) \leq \MMLab(X_{new})$, or
		  \item $X_{new} = X_{old}$, in which case it trivially holds 
			that $\MMLab(X_{old}) \leq \MMLab(X_{new})$.
		\end{enumerate}
		In either case, we obtain that
		$\MMLab(X_{old}) \leq \MMLab(X_{new})$.
		Furthermore, as it holds that \\
		$\MMLab(Y_{old}) = \MMLab(X_{old}) + 1$ (line 26) \\
		$\MMLab(Y_{new}) = \MMLab(X_{new}) + 1$ (line 26) \\
		$\MMLab(Z_{old}) = \MMLab(Y_{old}) + 1$ (line 32) \\
		$\MMLab(Z_{new}) = \MMLab(Y_{new}) + 1$ (line 32) \\
		it follows that $\MMLab(Z_{old}) \leq \MMLab(Z_{new})$.
	  \item $Z_{old}$ has been added by the for loop of lines 9-18. 
		In that case, it holds that $\MMLab(Z_{old}) = 1$ (line 15). 
		As $\MMLab(Z_{new}) \geq 1$ (Lemma 3) it directly follows 
		that $\MMLab(Z_{old}) \leq \MMLab(Z_{new})$.
	\end{enumerate}
\end{description}
\end{proof}

\begin{lemma} \label{lemma-a1-sorted-removing}
Let $\AF = (\Arguments, \attack)$ be an argumentation framework 
and $A$ an argument in the grounded extension of $\AF$. Let both $\AF$ 
and $A$ be given as input to Algorithm \ref{alg-construct}. The order in which 
arguments are removed from $\mathtt{unproc\_in}$ is non-descending w.r.t. 
$\MMLab$.  That is, if argument $x_1$ is removed from $\mathtt{unproc\_in}$
before argument $x_2$ is removed from $\mathtt{unproc\_in}$, then
$\MMLab(x_1) \leq \MMLab(x_2)$.
\end{lemma}

\begin{proof}
This follows directly from Lemma \ref{lemma-a1-sorted-adding}, together with 
the fact that additions to and removals from $\mathtt{unproc\_in}$ are done 
according to the FIFO (First In First Out) principle.
\end{proof}

We proceed to show the correctness of $\MMLab$ in an inductive way.
That is, we show that $\MMLab$ is correct at the start of each iteration
of the while loop. We then later need to do a bit of additional work to
state that $\MMLab$ is still correct at the moment we jump out of the
while loop using the return statement.
    
\begin{lemma} \label{lemma-a1-mm-correct}
Let $AF = (\Arguments, \attack)$ be an argumentation framework 
and $A$ an argument in the grounded extension of $\AF$. Let both $\AF$ 
and $A$ be given as input to Algorithm \ref{alg-construct}. At the start of 
each iteration of the while loop, it holds that $\MMLab$ is a correct min-max
numbering of $\Lab$.
\end{lemma}

\begin{proof}
We prove this by induction over the number of loop iterations.

As for the basis of the induction (n=1), let us consider the first loop 
iteration. This is just after the for loop of lines 9-18 has finished. We need 
to prove that $\MMLab$ is a correct min-max numbering of $\Lab$ According to 
the definition of a min-max numbering (Definition \ref{def-min-max-numbering}) 
we need to prove that for every $x$ in $\Arguments$:
\begin{enumerate}
  \item if $\Lab(x) = \inn$ then $\MMLab(x) = 
	max(\{\MMLab(y) | y$ attacks $x$ and $\Lab(y) = \out\}) + 1$ \\
	Suppose $\Lab(x) = \inn$. This means that $x$ has been labelled $\inn$
	by the for loop of lines 9-18, which implies that $x$ does not have
	any attackers and is numbered 1. That is, $\MMLab(x) = 1$ and
	$max(\{\MMLab(y) | y$ attacks $x$ and $\Lab(y) = \out\}) = 0$
	(by definition). Therefore $\MMLab(x) = 
	max(\{\MMLab(y) | y$ attacks $x$ and $\Lab(y) = \out\}) + 1$
  \item if $\Lab(x) = \out$ then $\MMLab(x) =
	min(\{\MMLab(y) | y$ attacks $x$ and $\Lab(y) = \inn\}) + 1$ \\
	This is trivially the case, as at the end of the for loop (lines 9-18)
	no argument is labelled out.
\end{enumerate}
   
As for the induction step, suppose that at the start of a particular loop
iteration, $\MMLab$ is a correct min-max numbering of $\Lab$. We need to 
prove that if there is a next loop iteration, then at the start of this
next loop iteration it is still the case that $\MMLab$ is a correct min-max
numbering of $\Lab$. For this, we need to prove that at the end of the current 
loop iteration, for any $x \in \Arguments$ it holds that:
\begin{enumerate}
  \item if $\Lab(x) = \inn$ then $\MMLab(x) =
	max(\{\MMLab(y) | y$ attacks $x$ and $\Lab(y) = \out\})+1$ \\
	We distinguish two cases:
	\begin{enumerate}
	  \item $x$ was already labelled $\inn$ at the start of the current 
		loop iteration. Then, as $\Lab$ is an admissible labelling at 
		each point of the algorithm (Lemma \ref{lemma-a1-admissible})
		each attacker $y$ of $x$ is labelled $\out$ by $\Lab$. These 
		attackers are still labelled $\out$ at the end of the current
		loop iteration (once an argument is labelled $\out$, it stays 
		labelled $\out$). Also, the value $\MMLab(y)$ of these $\out$
		labelled attackers remains unchanged. Hence, from the fact that 
		$\MMLab(x) = max(\{\MMLab(y) | y$ attacks $x$ and $\Lab(y)
		= \out\})+1$ at the start of the current iteration, it follows 
		that $\MMLab(x) = max(\{\MMLab(y) | y$ attacks $x$ and 
		$\Lab(y) = \out\}) + 1$ at the end of the current iteration.
	  \item $x$ became labelled $\inn$ during the current loop iteration.
		In that case, $x$ was labelled $\inn$ at line 31 (with $Z=x$).
		So $Z=x$ in $\MMLab(Z) = \MMLab(Y) + 1$ (line 32).
		We therefore need to show that $\MMLab(Y) =
		max(\{\MMLab(y) | y$ attacks $Z$ and $\Lab(y) = \out\})$.
		As $Y$ is an $\out$ labelled attacker of $Z$, we already know 
		that $max(\{\MMLab(y) | y$ attacks $Z$ and $\Lab(y) = \out\})$
		will be \emph{at least} $\MMLab(Y)$. We now proceed to show 
		that $max(\{\MMLab(y) | y$ attacks $Z$ and $\Lab(y) = \out\})$
		will be \emph{at most} $\MMLab(Y)$. That is, for each $\out$
		labelled attacker $y$ of $Z$ we show that 
		$\MMLab(y) \leq \MMLab(Y)$.
		Let $Y'$ be an arbitrary $\out$ labelled attacker of $Z$. 
		Let $X$ be the $\inn$ labelled attacker of $Y$ (line 22 of the 
		current loop iteration) and let $X'$ be the $\inn$ labelled 
		attacker of $Y'$ (line 22 of the current or a previous loop 
		iteration). We distinguish two cases:
		\begin{itemize}
		  \item $X' = X$ \\
			In that case, from the fact that
			$\MMLab(Y) = \MMLab(X) + 1$ (line 26) and
			$\MMLab(Y') = \MMLab(X') + 1$ (line 26)
			it follows that $\MMLab(Y') = \MMLab(Y)$ so 
			(trivially) also that $\MMLab(Y') \leq \MMLab(Y)$.
		  \item $X' \neq X$ \\
			As $X$ was removed from $\mathtt{unproc\_in}$ during 
			the current loop iteration, it follows that $X'$ was 
			removed from $\mathtt{unproc\_in}$ during one of the 
			previous loop iterations. This means that $X'$ was 
			removed from $\mathtt{unproc\_in}$ \emph{before} $X$
			was removed from $\mathtt{unproc\_in}$, which implies
			(Lemma \ref{lemma-a1-sorted-removing}) that
			$\MMLab(X') \leq \MMLab(X)$. From the fact that
			$\MMLab(Y) = \MMLab(X) + 1$ (line 26) and
			$\MMLab(Y') = \MMLab(X') + 1$ (line 26)
			it follows that $\MMLab(Y') \leq \MMLab(Y)$.
		\end{itemize}
		As we now observed that $\MMLab(x)$ is the correct min-max 
		number of $x$ at the moment it was assigned (line 32) we can
		use similar reasoning as at the previous point (point (a)) to 
		obtain that it is still the correct min-max number at the end
		of the current loop iteration.
	\end{enumerate}
  \item if $\Lab(x) = \out$ then $\MMLab(x) = min(\{\MMLab(y) | y$
	attacks $x$ and $\Lab(y) = \inn\})+1$ \\
	We distinguish two cases:
	\begin{enumerate}
	  \item $x$ was already labelled $\out$ at the start of the current 
		loop iteration. In that case, our induction hypothesis that
		the min-max numbers are correct at the start of the current
		loop iteration implies that $\MMLab(x) = 
		min(\{\MMLab(y) | y$ attacks $x$ and $\Lab(y) = \inn\})+1$
		at the start of the current loop iteration. As the current
		loop iteraton does not change the value of $\MMLab(x)$ (once 
		a value for $\MMLab(x)$ is assigned, the algorithm never
		changes it) this value will still be the same at the end of
		the current loop iteration. We therefore only need to verify 
		that this value is still correct at the end of the current loop
		iteration. For this, we need to be sure that any newly $\inn$ 
		labelled argument (that is, an argument that became labelled 
		$\inn$ during the current loop iteration) does not change the 
		value of $min(\{\MMLab(y) | y$ attacks $x$ and $\Lab(y) = 
		\inn\})$. Let $Z$ be a newly $\inn$ labelled attacker of $x$ 
		(line 31).  Then $Z$ was added to the rear of 
		$\mathtt{unproc\_in}$ (line 30). Let $Z'$ be an arbitrary 
		$\inn$ labelled attacker of $x$. We distinguish two cases:
		\begin{itemize}
		  \item $Z' = Z$ \\
			In that case, it directly follows that
			$\MMLab(Z') = \MMLab(Z)$ so (trivially) also that 
			$\MMLab(Z') \leq \MMLab(Z)$.
		  \item $Z' \neq Z$ \\
			In that case, it follows that $Z'$ was added to 
			$\mathtt{unproc\_in}$ \emph{before} $Z$ was added
			to $\mathtt{unproc\_in}$. 
			Lemma \ref{lemma-a1-sorted-adding} then implies that
			$\MMLab(Z') \leq \MMLab(Z)$.
		\end{itemize}
		In both cases, we obtain that $\MMLab(Z') \leq \MMLab(Z)$.
		This means that whenever $x$ gets a new $\inn$ labelled 
		attacker $min(\{\MMLab(y) | y$ attacks $x$ and 
		$Lab(y) = \inn\})$ does not change. Therefore, the value of 
		$\MMLab(x)$ is still the correct min-max number of $x$ at 
		the end of the current loop iteration.
	  \item $x$ became labelled $\out$ during the current loop iteration.
		This can only have happened at line 25, so $x=Y$. 
		$\MMLab(Y)$ is then assigned $\MMLab(X)+1$ at line 26. 
		In order for $\MMLab(Y)$ to be a correct min-max number, 			we need to verify that $\MMLab(X) = min(\{\MMLab(y) | y$
		attacks $Y$ and $\Lab(y) = \inn\})$.
		This is the case because at line 25, $X$ is the only $\inn$
		labelled attacker of $Y$ (otherwise $Y$ would have been 
		labelled $\out$ before). As we have observed that $\MMLab(Y)$
		is the correct min-max value at the moment it was assigned,
		we can use similar reasoning as at the previous point 
		(point (a)) to obtain that it is still the correct min-max
		number at the end of the current loop iteration.
	\end{enumerate}
\end{enumerate}
\end{proof}

In order for a labelling to be strongly admissible, its min-max numbering
has to contain natural numbers only (no $\infty$). We therefore proceed to
show the absence of $\infty$ in an inductive way. That is, we show the 
absence of $\infty$ at the start of each iteration of the while loop. We then 
later need to do a bit of additional work to show the absence of $\infty$ at 
the moment we jump out of the while loop using the return statement.

\begin{lemma} \label{lemma-a1-no-infty}
Let $\AF = (\Arguments, \attack)$ be an argumentation framework and 
let $A$ be an argument in the grounded extension of $\AF$. Let both $\AF$ and
$A$ be given as input to Algorithm \ref{alg-construct}. At the start of each 
iteration of the while loop at lines 21-37, it holds that for each $\inn$ or 
$\out$ labelled argument $x \in \Arguments$, $\MMLab(x)$ is a natural number 
(no $\infty$)
\end{lemma}

\begin{proof}
We prove this by induction over the number of iterations of the while loop 
at lines 21-37. \newline
As for the basis of induction(n=1), let us consider the first loop iteration. 
This is just after the for loop at lines 9-18 has finished. We need to prove 
that for each $\inn$ or $\out$ labelled argument $x$ $\in Ar$, $\MMLab(x)$ is 
a natural number. We therefore need to prove that:
\begin{enumerate}
  \item if $\Lab(x) = \inn$ then $\MMLab(x) \neq \infty$ \\
	Let $x$ be labelled $\inn$ by the for loop at lines 9-18. 
	This can only have happened at line 14. According to line 15, 
	it then follows that $\MMLab(x) = 1$. Hence $\MMLab(x) \neq \infty$.
  \item if $\Lab(x) = \out$ then $\MMLab(x) \neq \infty$ \\
	This is trivially the case as the end of the for loop at lines 9-18, 
	no argument is labelled $\out$. 
\end{enumerate}
As for the induction step, suppose that at the start of a particular loop 
iteration, for each $\inn$ or $\out$ labelled argument $x \in Ar$, $\MMLab(x)$
is a natural number. We therefore need to prove that by the end of the 
iteration (and therefore also at the start of the next loop iteration) 
it holds that:
\begin{enumerate}
  \item if $\Lab(x) = \inn$ then $\MMLab(x) \neq \infty$ \\
	We distinguish two cases:
	\begin{itemize}
	  \item $x$ was already labelled $\inn$ at the start of the current 
		loop iteration. From the induction hypothesis it follows that 
		$\MMLab(x) \neq \infty$ at the start of the current iteration.
		As Algorithm \ref{alg-construct} does not change any values of
		$\MMLab$ once these have been assigned, it follows that 
		$\MMLab(x) \neq \infty$ will hold at the end of the current 
		loop iteration.
	  \item $x$ became labelled $\inn$ during the current loop iteration.
		In the case $x$ was labelled $\inn$ at line 31 (with $Z = x$).
		Following line 32, $\MMLab(X) = \MMLab(Y) + 1$. According to
		line 26, $\MMLab(Y) = \MMLab(X) + 1$ with X being an attacker
		of Y that became labelled $\inn$ during a previous iteration of
		the while loop. From our induction hypothesis it follows that
		$\MMLab(X) \neq \infty$. As $\MMLab(Y) = \MMLab(X)+1$ it
		follows that $\MMLab(Y) \neq \infty$. As $\MMLab(Z) = 
		\MMLab(Y)+1$ it follows that $\MMLab(Z) \neq \infty$.
		That is (as $x = Z$) $\MMLab(x) \neq \infty$.
	\end{itemize}
  \item if $\Lab(x) = \out$ them $\MMLab(x) \neq \infty$ \\
	We distinguish two cases:
	\begin{itemize}
	  \item $x$ was already labelled $\out$ at the start of the current 
		loop iteration. From the induction hypothesis it follows that 
		$\MMLab(x) \neq \infty$ at the start of the current iteration.
		As Algorithm \ref{alg-construct} does not change any values of 
		$\MMLab$ once these have been assigned, it follows that 
		$\MMLab(x) \neq \infty$ will hold at the end of the current 
		loop iteration.
	  \item $x$ became labelled $\out$ during the current loop iteration.
		This can only have happened at line 25 (with $x = Y$). 
		$\MMLab(Y)$ is then assigned $\MMLab(X) + 1$ at line 26, 
		with $X$ being an attacker of $Y$ that became labeled $\inn$ 
		during a previous iteration of the while loop. From the 
		induction hypothesis, it follows that $\MMLab(X) \neq \infty$. 
		As $\MMLab(Y) = \MMLab(X)+1$ (line 26) it follows that 
		$\MMLab(Y) \neq \infty$. That is (as $x = Y$) $\MMLab(x) 
		\neq \infty$.
	\end{itemize}
\end{enumerate}
\end{proof}

Although most of our results so far are about the algorithm itself,
we also need an additional theoretical property of grounded semantics,
stated in the following lemma.

\begin{lemma} \label{lemma-grounded}
Let $\AF = (\Arguments, \attack)$ be an argumentation framework. It holds 
that the grounded labelling of $\AF$ is the only argument labelling that is 
both strongly admissible and complete.
\end{lemma}

\begin{proof}
First of all, it has been observed that the grounded labelling is both strongly 
admissible \cite{CD19a} and complete (Definition \ref{def-lab-semantics-ii}).
We proceed to prove that it is also the \emph{only} argument labelling that is 
both strongly admissible and complete. Let $\Lab$ be an argument labelling that 
is both strongly admissible and complete. From the fact that the grounded 
labelling ($\Lab_{gr}$) is the unique biggest strongly admissible labelling 
\cite{CD19a} it follows that $\Lab \sqsubseteq \Lab_{gr}$.
From the fact that the grounded labelling is the unique smallest complete
labelling (Definition \ref{def-lab-semantics-ii}) it follows that 
$\Lab_{gr} \sqsubseteq \Lab$. Together, this implies that $\Lab = \Lab_{gr}$.
\end{proof}

If we would not finish the algorithm after hitting the main argument and
instead continue to execute the algorithm until $\mathtt{unproc\_in}$ is
empty, we would be computing the grounded labelling with its associated
min-max numbering as stated by the following lemma.

\begin{lemma} \label{lemma-a1-modified-grounded}
If in Algorithm \ref{alg-construct} one would comment out line 16 and line 33 
and add the following line (line 41) at the end:\\
return $\Lab$ and $\MMLab$ \\
then the output of the thus modified algorithm would be the grounded 
labelling $\Lab$ of $\AF$, together with its min-max numbering $\MMLab$.
\end{lemma}

\begin{proof}
We first observe that $\Lab$ is a strongly admissible labelling. 
This follows from the facts that 
\begin{enumerate}
  \item $\Lab$ is an admissible labelling\\
	This can be proved in a similar way as Lemma \ref{lemma-a1-admissible}.
  \item $\MMLab$ is a correct min-max numbering of $\Lab$\\
	This can be proved in a similar way as Lemma \ref{lemma-a1-mm-correct}.
  \item $\MMLab$ does not contain $\infty$ (natural numbers only)\\
	This can be proved in a similar way as \ref{lemma-a1-no-infty}.
\end{enumerate}

We proceed to show that $\Lab$ is also a complete labelling.
For this, we first show the following two properties:
\begin{enumerate}
  \item if $\Lab(y) = \out$ for each attacker $y$ of $x$ then $\Lab(x) = \inn$\\
	Suppose $\Lab(y) = \out$ for each attacker of $x$. This means that at 
	the end of the algorithm, it holds that $\mathtt{undec\_pre}(x) = 0$ 
	which implies that $x$ became labelled $\inn$ (either at line 14 or 
	at line 31) at the moment when $\mathtt{undec\_pre}(x)$ became $0$ 
	(at either line 11 or line 28)
  \item if $\Lab(y) = \inn$ for some attacker $y$ of $x$ then $\Lab(x) = \out$\\
	Suppose $\Lab(y) = \inn$ for some attacker $y$ of $x$. At the end 
	of the algorithm, it holds that $\mathtt{unproc\_in}$ is empty. 
	As each $\inn$ labelled argument in $\Lab$ (such as $y$) was added 
	to $\mathtt{unproc\_in}$ when it became labelled $\inn$, this implies 
	that each $\inn$ labelled argument in $\Lab$ (in particular $y$) was 
	subsequently removed from $\mathtt{unproc\_in}$. This removal can 
	only have happened at line 23, which implies (line 24 and 25) that
	each argument that is attacked by $y$ (in particular $x$) is labelled 
	$\out$.
\end{enumerate}
Suppose $\Lab(x) = \undec$. From point 1 and the fact that $\Lab(x) \neq \inn$
we obtain that
(3) there is an attacker $y$ of $x$ such that $\Lab(y) \neq \out$. 
From point 2 and the fact that $\Lab(x) \neq \out$ we obtain that
(4) there is no attacker $y$ of $x$ such that $\Lab(y) = \inn$.
From point (3) and (4) it follows that
\begin{itemize}
  \item if $\Lab(x) = \undec$ then there is a $y$ that attacks $x$ such that
	$\Lab(y) = \undec$ and for each $y$ that attacks $x$ such that
	$\Lab(y) \neq \undec$ it holds that $\Lab(y) = \out$
\end{itemize}
This, together with the fact that $\Lab$ is an admissible labelling implies
that $\Lab$ is a complete labelling (Definition \ref{def-lab-semantics-i}).

From the thus obtained facts that $\Lab$ is both a strongly admissible 
labelling and a complete labelling it follows (Lemma \ref{lemma-grounded})
that $\Lab$ is the grounded labelling.
\end{proof}

Using the above lemmas, we now proceed to show the correctness of the algorithm.

\begin{theorem} \label{th-a1-correct}
Let $\AF = (\Arguments, \attack)$ be an argumentation framework and 
let $A$ be an argument in the grounded extension of $\AF$. Let both $\AF$ 
and $A$ be given as input to Algorithm \ref{alg-construct}. Let $\Lab$ and 
$\MMLab$ be the output of the algorithm. It holds that $\Lab$ is a strongly 
admissible labelling that labels $A$ $\inn$ and has $\MMLab$ as its min-max 
numbering.
\end{theorem}

\begin{proof}
We first observe that as $A$ is in the grounded extension of $\AF$, the 
modified algorithm of Lemma \ref{lemma-a1-modified-grounded} would have 
produced the grounded labelling, which labels $A$ in. This implies that
at some moment in Algorithm \ref{alg-construct}, line 16 or line 33 is 
trigered, meaning that $\Lab$ as returned by Algorithm \ref{alg-construct}
labels $A$ $\inn$. At the moment the return statement of line 16 or 33 is 
triggered, it holds that:
\begin{enumerate}
  \item $\Lab$ is an admissible labelling.\\
	This follows directly from Lemma \ref{lemma-a1-admissible}.
  \item $\MMLab$ is a correct min-max numbering of $\Lab$.\\
	To see that this is the case, we distinguish two cases:
	\begin{enumerate}
	  \item The return statement that was triggered was the one at line 16.
		In that case, $\MMLab$ is the correct min-max numbering
		of $\Lab$. The proof is similar to the first half of
		the proof of Lemma \ref{lemma-a1-mm-correct}.
	  \item The return statement that was triggered was the one at line 33.
		In that case, Lemma \ref{lemma-a1-mm-correct} tells us
		that the value of $\MMLab$ at the start of the last iteration
		of the while loop was a correct min-max numbering of the
		value of $\Lab$ at the start of the last iteration of the
		while loop. We then need to show that the value of
		$\MMLab$ at the time of the return statement (line 33)
		is still a correct min-max numbering of the value of $\Lab$
		at the time of the return statement (line 33). This can be
		proved in a similar way as is done in the second half of
		the proof of Lemma \ref{lemma-a1-mm-correct} (instead of
		going until the end of the loop iteration, one goes until
		the moment the return statement of line 33 is triggered).
	\end{enumerate}
  \item $\MMLab$ numbers each $\inn$ or $\out$ labelled argument with a 
	natural number (no $\infty$).\\
	To see that this is the case, we distinguish two cases:
	\begin{enumerate}
	  \item The return statement that was triggered is the one at line 16.
		In that case, for each $\inn$ or $\out$ labelled argument $x$
		it holds that $\MMLab(x) \neq \infty$. The proof is similar to
		the first half (the basis) of the proof of Lemma 
		\ref{lemma-a1-no-infty}.
	  \item The return statement that was triggered is the one at line 33.
		In that case, Lemma \ref{lemma-a1-no-infty} tells us that at
		the start of the last iteration of the while loop, $\MMLab(x)
		\neq \infty$ for each argument $x$ that was labelled $\inn$ or
		$\out$. We need to show that this is still the case at
		the time of the return statement (line 33). This can be
		proved in a similar was as is done in the second half of the
		proof of Lemma \ref{lemma-a1-no-infty} (instead of going until
		the end of the loop iteration, the idea is to go until the
		return statement of line 33 is triggered).
	\end{enumerate}
\end{enumerate}
\end{proof}

It turns out that the algorithm runs in polynomical time (more specific, in
cubic time).

\begin{theorem} \label{th-a1-polynomial}
Let $\AF = (\Arguments, \attack)$ be an argumentation framework and 
let $A$ be an argument in the grounded extension of $\AF$. Let both $\AF$ and 
$A$ be given as input to Algorithm \ref{alg-construct}. Let $\Lab$ and $\MMLab$
be the output of the algorithm. It holds that Algorithm \ref{alg-construct}
computes $\Lab$ and $\MMLab$ in O$(n)^3$ time
\end{theorem}

\begin{proof}
Let $n$ be the number of arguments in $\AF$ (that is, $n = |\Arguments|$).
The for loop (lines 9-18) can have at most $n$ iterations. The while loop
(lines 21-37) can also have at most $n$ iterations. This is because each
iteration of the while loop removes an argument from $\mathtt{unproc\_in}$,
which can be done $n$ times at most, given that no argument can be added to
$\mathtt{unproc\_in}$ more than once (this follows from line 31 and line 27).
For each iteration of the while loop, the outer for loop (lines 24-36) will
run at most $n$ times. Also, for each iteration of the outer for loop, the
inner for loop (lines 27-35) will run at most $n$ times. This means that the
total number of instructions executed by the while loop is of the order $n^3$
at most. This, combined with the earlier observed fact that the for loop of
lines 9-18 runs at most $n$ times brings the total complexity of Algorithm 
\ref{alg-construct} to $O(n + n^3) = O(n^3)$.
\end{proof}

%% file: sec03b-algorithm.tex
\subsection{Algorithm 2} \label{subsec-prune}

The basic idea of Algorithm \ref{alg-prune} is to prune the part of the strongly
admissible labelling that is not needed, by identifying the part that actually 
is needed.  This is done in a top-down way, starting by including the main 
argument (which is labelled $\inn$), then including all its attackers (which 
are labelled $\out$), for each of which a minimally numbered $\inn$ labelled 
attacker is included, etc. The idea is to keep doing this until reaching the 
($\inn$ labelled) arguments that have no attackers. Each argument that has not 
been included by this process is unnecessary for the strongly admissible 
labelling and can be made $\undec$, resulting in a labelling that is smaller 
or equal to the strongly labelling labelling one started with.

\begin{algorithm}
\caption{Prune a strongly admissible labelling that labels $A$ $\inn$
and its associated min-max numbering.} \label{alg-prune}
\begin{algorithmic}[1]
\Statex \textbf{Input}: An argumentation framework $\AF=(\Arguments,\attack)$,
\Statex an argument $A \in \Arguments$ that is in the grounded extension 
	of $\AF$, A strongly admissible labelling $\Lab_I$ where $A \in \inn(\Lab_I)$ and the associated min-max numbering $\MMLabI$.
\Statex \textbf{Output}: A strongly admissible labelling $\Lab_O \sqsubseteq \Lab_I$ where $A \in \inn(\Lab_O)$,
\Statex the associated min-max numbering $\MMLabO$.
\Statex

\State  // We start with the type definitions
\State  $\Lab_O: \Arguments \rightarrow \{\inn, \out, \undec\}$
\State  $\MMLabO: \inn(\Lab) \cup \out(\Lab) \rightarrow \mathbb{N} \cup \{\infty\}$
\State  $\mathtt{unproc\_in}: [X_1, ... X_n]$
	($X_i \in \Arguments$ for each $1 \leq i \leq n$)
	// list of arguments
\State  // Initialize $\Lab_O$ and include the main argument
\State  $\Lab_O \leftarrow (\emptyset, \emptyset, \Arguments)$
	// $\Lab_O$ becomes the all-$\undec$ labelling
\State  $\mathtt{unproc\_in} \leftarrow [A]$
\State  $\Lab_O(A) \leftarrow \inn$
\State  $\MMLabO(A) \leftarrow \MMLabI(A)$
\State
\State // Next, process the other arguments in a top-down way
        \While{$\mathtt{unproc\_in}$ is not empty}
\State      let $X$ be the argument at the front of $\mathtt{unproc\_in}$
\State      remove $X$ from $\mathtt{unproc\_in}$
            \For{each attacker $Y$ of $X$}
\State          $\Lab_O(Y) \leftarrow \out$
\State          $\MMLabO(Y) \leftarrow \MMLabI(Y)$
		\If{there is no minimal (w.r.t $\MMLabI$) $\inn$ labelled 
		(w.r.t $\Lab_{I}$) attacker of $Y$ that is also labelled $\inn$
		by $Lab_O$}
\State              Let $Z$ be a minimal (w.r.t $\MMLabI$) in labelled 
		    (w.r.t $Lab_I$) attacker of $Y$
\State		    Add Z to the rear of $\mathtt{unproc\_in}$
\State		    $\Lab_O(Z) \leftarrow \inn$
\State		    $\MMLabO(Z) \leftarrow \MMLabI(Z)$
		\EndIf
            \EndFor
        \EndWhile
\end{algorithmic}
\end{algorithm}

To see how the algorithm works, consider again the argumentation framework
of Figure \ref{fig-example-AF}. Let $C$ be the main argument. Suppose the
input labelling $\Lab_I$ is $(\{A,C,D\}, \{B\}, \{E,F,G,H\})$ and its
associated input labelling numbering $\MMLabO$ is $\{(A:1), (B:2), (C:3),
(D:1)\}$.\footnote{The reader may have noticed that this was the output of
Algorithm \ref{alg-construct} for the example that was given in Section
\ref{subsec-construct}.} At the start of the first iteration of the while
loop, it holds that $\Lab_O = (\{C\}, \emptyset, \{A,B,D,E,F,G,H\})$,
$\MMLabO = \{(C:1)\}$ and $\mathtt{unproc\_in} = [C]$.
The first iteration of the while loop then removes $C$ from 
$\mathtt{unproc\_in}$ (line 14), labels its attacker $B$ $\out$ (line 16),
numbers $B$ with $2$ (line 17), adds $A$ to $\mathtt{unproc\_in}$ (line 20),
labels $A$ $\inn$ (line 21) and numbers $A$ with $1$ (line 22).
The second iteration of the while loop then removes $A$ from 
$\mathtt{unproc\_in}$ (line 14). However, as $A$ does not have any attackers,
the for loop (lines 15-24) is skipped. As $\mathtt{unproc\_in}$ is now empty,
the while loop is finished and the algorithm terminates, with $\Lab_O =
(\{A,C\}, \{B\}, \{D,E,F,G,H\})$ and $\MMLabO = \{(A:1), (B:2), (C:3)\}$
being its results.

We now proceed to prove some of the formal properties of the algorithm.
The first property to be proved is termination.

\begin{theorem} \label{th-a2-terminates}
Let $\AF = (\Arguments, \attack)$ be an argumentation framework, $A$ be an 
argument in the grounded extension of $\AF$, $\Lab_I$ be a strongly admissible 
labelling where $A$ is labelled in and $\MMLabI$ be the associated min-max 
numbering. Let $\AF$, $A$, $\Lab_I$ and $\MMLabI$ be given as input to 
Algorithm \ref{alg-prune}. It holds that the algorithm terminates.
\end{theorem}

\begin{proof}
At the while loop of lines 12-25, we observe that only a finite number of 
arguments can be added to \texttt{unproc\_in}. This is because there are only 
a finite number of arguments in the argumentation framework, and because no 
argument can be added to \texttt{unproc\_in} more than once. The latter can 
be seen as follows. Following line 18, only arguments that are not already 
labelled $\inn$ by $\Lab_O$ can be added to \texttt{unproc\_in}. Also, if an 
argument is labelled $\inn$ by $\Lab_O$, it will stay labelled $\inn$ by 
$\Lab_O$ as there is nothing in the algorithm that will change it. Following 
line 14, at each iteration of the while loop, an argument is removed from 
\texttt{unproc\_in}. From the fact that only a finite number of arguments 
can be added to \texttt{unproc\_in}, it directly follows that only a finite 
number of arguments can be removed from \texttt{unproc\_in}. Hence, the while 
loop can run only a finite number of times before \texttt{unproc\_in} is empty.
Hence, Algorithm \ref{alg-prune} terminates. 
\end{proof}

Next, we prove that the labelling that is yielded by the algorithm is
smaller or equal to the labelling the algorithm started with.

\begin{theorem} \label{th-LabO-LabI}
Let $\AF = (\Arguments, \attack)$ be an argumentation framework, $A$ be an 
argument in the grounded extension of $\AF$, $\Lab_I$ be a strongly admissible 
labelling where $A$ is labelled $\inn$ and $\MMLabI$ be the associated min-max 
numbering. Let $\AF$, $A$, $\Lab_I$ and $\MMLabI$ be given as input to 
Algorithm \ref{alg-prune}. Let $\Lab_O$ and $\MMLabO$ be the output of
Algorithm \ref{alg-prune}. It holds that $\Lab_O \sqsubseteq \Lab_I$
\end{theorem}

\begin{proof}
In order to prove that $\Lab_O \sqsubseteq \Lab_I$, we must show:
\begin{enumerate}
  \item $\inn(\Lab_O) \subseteq \inn(\Lab_I)$ \\
	Let $x$ be an arbitrary argument that is labelled $\inn$ by $\Lab_O$.
	We distinguish two cases:
	\begin{itemize}
          \item $x$ became labelled $\inn$ at line 8. Therefore, it follows 
		$x$ is the argument in question (with $x = A$). Therefore, 
		$A$ is also labelled $\inn$ by $\Lab_I$. That is, $x$ is also 
		labelled in by $\Lab_I$
	  \item $x$ became labelled $\inn$ at line 21. According to line 19, 
		$x$ is a minimal $\inn$ labelled attacker of some $\out$ 
		labelled argument $y$ w.r.t $\Lab_I$. Therefore, $x$ is also 
		labelled $\inn$ within $\Lab_I$. 
	\end{itemize}
  \item $\out(\Lab_O) \subseteq \out(\Lab_I)$ \\
	Let $y$ be an arbitrary $\out$ labelled argument within $\Lab_O$. 
	It follows that $y$ must have been labelled $\out$ at line 16
	(so $y = Y$). From line 15, it follows that $Y$ attacks $X$, which
	was removed from \texttt{unproc\_in} at line 14. This means that $X$
	at some point was added to \texttt{unproc\_in}, which could only have
	happened at line 7 or line 20. In either case, it holds that
	$\Lab_O(X) = \inn$ (line 8 or 21, respectively). From point 1 above,
	we infer that $\Lab_I(X) = \inn$. As $\Lab_I$ is an admissible
	labelling, it follows that each attacker of $X$ (such as $Y$) is
	labelled $\out$ by $\Lab_I$. As $y = Y$ it directly follows that
	$y$ is labelled $\out$ by $\Lab_I$.
\end{enumerate}
\end{proof}

Next, we prove that the output of the algorithm is at least admissible
(the fact that it is also \emph{strongly} admissible is proved further on).

\begin{theorem} \label{th-LabO-admissible}
Let $\AF = (\Arguments, \attack)$ be an argumentation framework, $A$ be an 
argument in the grounded extension of $\AF$, $\Lab_I$ be a strongly admissible 
labelling where $A$ is labelled $\inn$ and $\MMLabI$ be the associated min-max 
numbering. Let $\AF$, $A$, $\Lab_I$ and $\MMLabI$ be given as input to 
Algorithm \ref{alg-prune}. Let $\Lab_O$ and $\MMLabO$  be the output of 
Algorithm \ref{alg-prune}. It holds that $\Lab_O$ is an admissible labelling 
that labels $A$ $\inn$.
\end{theorem}

\begin{proof}
The fact that $\Lab_O$ labels $A$ $\inn$ follows from line 8. In order to 
prove that $\Lab_O$ is an admissible labelling, we must show that it satisfies 
the following two properties (Definition \ref{def-lab-semantics-i}):
\begin{enumerate}
  \item if $\Lab_O(x) = \inn$, then for each $y$ that attacks $x$ 
	it holds that $\Lab_O(y) = \out$ \\
	Let $x$ be an arbitrary $\inn$ labelled argument within $\Lab_O$.
	This means that $x$ became $\inn$ at line 8 or line 21. In either case,
	$x$ has been added to \texttt{unproc\_in} (at line 7 or line 20,
	respectively). Once the algorithm is terminated, \texttt{unproc\_in}
	has to be empty. This means that at some point, $x$ must have been
	removed from \texttt{unproc\_in}. This can only have happened at line
	14, which implies that (lines 15 and 16) each attacker $y$ of $x$ is
	labelled $\out$ by $\Lab_O$.
  \item if $\Lab_O(x) = \out$, then there exists a $y$ that attacks $x$
	such that $\Lab_O(y) = \inn$ \\
	Let $x$ be an arbitrary $\out$ labelled argument within $\Lab_O$.
	It follows that $x$ has been labelled $\out$ at line 16 ($x = Y$).
	According to Theorem \ref{th-LabO-LabI}, $Y$ is also labelled 
	$\out$ by $\Lab_I$. Since $\Lab_I$ is an admissible labelling of $\AF$,
	at least one of $Y$'s attackers is labelled $\inn$ by $\Lab_I$.
	Following lines 18-21, a minimal (w.r.t $\MMLabI$) $\inn$ labelled 
	attacker of $Y$ (w.r.t  $\Lab_I$), $y$ has been labelled $\inn$ by
	$\Lab_O$. That is, there exists a $y$ that attacks $x$ such that 
	$\Lab_O(y) = \inn$ 
\end{enumerate}
\end{proof}

Next, we prove that the algorithm does not change the min-max values of
the arguments it labels $\inn$ or $\out$.

\begin{lemma} \label{lemma-MMLabO-MMLabI}
Let $\AF = (\Arguments, \attack)$ be an argumentation framework, $A$ be an 
argument in the grounded extension of $\AF$, $\Lab_I$ be a strongly admissible 
labelling where $A$ is labelled $\inn$ and $\MMLabI$ be the associated min-max 
numbering. Let $\AF$, $A$, $\Lab_I$ and $\MMLabI$ be given as input to 
Algorithm \ref{alg-prune}. Let $\Lab_O$ and $\MMLabO$  be the output of 
Algorithm \ref{alg-prune}. It holds that for each argument $x$ that is 
labelled $\inn$ or $\out$ by $\Lab_O$, $\MMLabO(x) = \MMLabI(x)$.
\end{lemma}

\begin{proof}
This follows from Theorem \ref{th-LabO-LabI} and lines 9, 17 and 22 of 
Algorithm \ref{alg-prune}. 
\end{proof}

Next, we prove that the output numbering is actually the correct min-max
numbering of the output labelling.

\begin{theorem} \label{th-MMLabO-correct}
Let $\AF = (\Arguments, \attack)$ be an argumentation framework, $A$ be an 
argument in the grounded extension of $\AF$, $\Lab_I$ be a strongly admissible 
labelling where $A$ is labelled $\inn$ and $\MMLabI$ be the associated min-max 
numbering. Let $\AF$, $A$, $\Lab_I$ and $\MMLabI$ be given as input to 
Algorithm \ref{alg-prune}. Let $\Lab_O$ and $\MMLabO$  be the output of 
Algorithm \ref{alg-prune}. 
It holds that $\MMLabO$ is the correct min-max numbering of $\Lab_O$. 
\end{theorem}

\begin{proof}
Since $\Lab_O$ has been shown to be admissible (Theorem 
\ref{th-LabO-admissible}), we need to show that (Definition 
\ref{def-min-max-numbering}):
\begin{enumerate}
  \item if $\Lab_O(x) = \inn$ then 
	$\MMLabO(x) = max(\{\MMLabO(y) \mid y$ attacks $x$ and 
	$\Lab_O(y) = \out \}) + 1$ \\
	Let $x$ be an arbitrary $\inn$ labelled argument within $\Lab_O$. 
	According to Lemma \ref{lemma-MMLabO-MMLabI}, $\MMLabO(x) = 
	\MMLabI(x)$. Since $\MMLabI$ is the correct min-max numbering of 
	$\Lab_I$, $\MMLabI(x) = max(\{\MMLabI(y) \mid y$ attacks $x$ and 
	$\Lab_I(y) = \out \}) + 1$. It follows that $\MMLabO(x) = 
	max(\{\MMLabI(y) \mid y$ attacks $x$ and $\Lab_I(y) = \out \}) + 1$. 
	From the fact that $\Lab_O(x) = \inn$ and that $\Lab_O \sqsubseteq
	\Lab_I$ (Theorem \ref{th-LabO-LabI}) it follows that $\Lab_I(x) = 
	\inn$.  As both $\Lab_I$ and $\Lab_O$ are admissible labellings, it
	holds that in both labellings, all attackers of $x$ are labelled $\out$.
	It follows that $\{ y \mid y$ attacks $x$ and $\Lab_I(y) = \out \} =
	\{ y \mid y$ attacks $x$ and $\Lab_O(y) = \out \}$. From Lemma
	\ref{lemma-MMLabO-MMLabI}, it then follows that $\{ \MMLabI(y) \mid y$ 
	attacks $x$ and $\Lab_I(y) = \out \} = \{ \MMLabO(y) \mid y$ attacks 
	$x$ and $\Lab_O(y) = \out \}$. Therefore, from the earlier observed 
	fact that $\MMLabO(x) = max(\{\MMLabI(y) \mid y$ attacks $x$ and 
	$\Lab_I(y) = \out \}) + 1$ we obtain that $\MMLabO(x) = 
	max(\{\MMLabO(y) \mid y$ attacks $x$ and $\Lab_O(y) = \out \}) + 1$.
  \item if $\Lab_O(x) = \out$ then
	$\MMLabO(x) = min(\{\MMLabO(y) \mid y$ attacks $x$ and
	$\Lab_O(y) = \inn \}) + 1$ \\
	Let $x$ be an arbitrary $\out$ labelled argument within $\Lab_O$.
	As $\MMLabI$ is the correct min-max numbering of $\Lab_I$ it holds
	that $\MMLabI(x) = min(\{\MMLabI(y) \mid y$ attacks $x$ and
	$\Lab_I(y) = \inn \}) + 1$. As $\MMLabO(x) = \MMLabI(x)$
	(Lemma \ref{lemma-MMLabO-MMLabI}), it follows that $\MMLabO(x) = 
	min(\{\MMLabI(y) \mid y$ attacks $x$ and $\Lab_I(y) = \inn \}) + 1$.
	The fact that $\Lab_O(x) = \out$ means that $x$ must have become
	labelled out at line 16. From lines 18-22, it follows that $\Lab_O$
	will also contain a minimal (w.r.t. $\MMLabI$) $\inn$ labelled attacker
	(w.r.t. $\Lab_I$). This implies that $min(\{\MMLabI(y) \mid y$ attacks
	$x$ and $\Lab_I(y) = \inn \}) = min(\{\MMLabO(y) \mid y$ attacks
	$x$ and $\Lab_O(y) = \inn \})$. So from the earlier obtained fact that
	$\MMLabO(x) = min(\{\MMLabI(y) \mid y$ attacks $x$ and $\Lab_I(y) = 
	\inn \}) + 1$, it follows that $\MMLabO(x) = min(\{\MMLabO(y) \mid y$
	attacks $x$ and $\Lab_O(y) = \inn \}) + 1$.
\end{enumerate}
\end{proof}

We are now ready to state one of the main results of the current section:
the output labelling is strongly admissible.

\begin{theorem} \label{th-LabO-strongly-admissible}
Let $\AF = (\Arguments, \attack)$ be an argumentation framework, $A$ be an 
argument in the grounded extension of $\AF$, $\Lab_I$ be a strongly admissible 
labelling where $A$ is labelled $\inn$ and $\MMLabI$ be the associated min-max 
numbering. Let $\AF$, $A$, $\Lab_I$ and $\MMLabI$ be given as input to 
Algorithm \ref{alg-prune}. Let $\Lab_O$ and $\MMLabO$  be the output of 
Algorithm \ref{alg-prune}. It holds that $\Lab_O$ is a strongly admissible 
labelling of $\AF$.
\end{theorem}

\begin{proof}
In order to show that $\Lab_O$ is strongly admissible, we need to show that 
$\Lab_O$ is an admissible labelling for which the min-max numbering does not 
contain any $\infty$ (Definition \ref{def-strong-adm}). First, we observe that 
$\Lab_O$ is an admissible labelling of $\AF$ (Theorem \ref{th-LabO-admissible})
with $\MMLabO$ as its correct min-max numbering (Theorem 
\ref{th-MMLabO-correct}). As $\Lab_I$ is a strongly admissible labelling of
$\AF$, its min-max numbering does not contain $\infty$. This, together with 
the fact that $\Lab_O \sqsubseteq \Lab_I$ (Theorem \ref{th-LabO-LabI}) and 
the fact that for each $\inn$ or $\out$ labelled argument $x$ by $\Lab_O$, $x$
is assigned the same min-max numbering by $\MMLabO$ as by $\MMLabI$ 
(Lemma \ref{lemma-MMLabO-MMLabI}) implies that $\MMLabO$ does not contain 
any $\infty$. Hence, we observe that $\Lab_O$ is an admissible labelling 
whose min-max numbering $\MMLabO$ does not contain $\infty$. That is, 
$\Lab_O$ is a strongly admissible labelling of $\AF$.
\end{proof}

It turns out that the algorithm runs in polynomial time (more specific, in
cubic time).

\begin{theorem} \label{th-a2-polynomial}
Let $\AF = (\Arguments, \attack)$ be an argumentation framework, $A$ be an 
argument in the grounded extension of $\AF$, $\Lab_I$ be a strongly admissible 
labelling where $A$ is labelled in and $\MMLabI$ be the correct min-max 
numbering of $\Lab_I$. Let $\AF$, $A$, $\Lab_I$ and $\MMLabI$ be given as 
input to Algorithm \ref{alg-prune}. Let $\Lab_O$ and $\MMLabO$ be the output 
of Algorithm \ref{alg-prune}. It holds that Algorithm \ref{alg-prune} computes 
$\Lab_O$ and $\MMLabO$ in O$(n)^3$ time. 
\end{theorem}

\begin{proof}
Let $n$ be the number of arguments in $\AF$ (that is, $n = |\Arguments|$). The while loop (lines 12-25) can have at most $n$ iterations. This is because each iteration of the while loop removes an argument from $\mathtt{unproc\_in}$, which can be done $n$ times at most, given that no argument can be added to $\mathtt{unproc\_in}$ more than once (this follows from lines 18-21). For each iteration of the while loop, the for loop (lines 15-24) will run at most $n$ times. In addition, for each iteration of the for loop, a sequential search (lines 18-19) will run at most $n$ times. This means that the total number of instructions executed by the while loop is of the order $n^3$ at most. Therefore, Algorithm \ref{alg-prune} computes $\Lab_O$ in O$(n)^3$ time.
\end{proof}

%% file: sec03c-algorithm.tex
\subsection{Algorithm \ref{alg-combine}} \label{subsec-combine}

The idea of Algorithm \ref{alg-combine} is to combine Algorithm 
\ref{alg-construct} and Algorithm \ref{alg-prune}, by running them in sequence.
That is, the output of Algorithm \ref{alg-construct} is used as input for 
Algorithm \ref{alg-prune}.

\begin{algorithm}
\caption{Construct a relatively small strongly admissible labelling that labels
	$A$ $\inn$ and its associated min-max numbering.} \label{alg-combine}
\begin{algorithmic}[1]
\Statex \textbf{Input}: An argumentation framework $\AF=(\Arguments,\attack)$,
\Statex an argument $A \in \Arguments$ that is in the grounded extension 
	of $\AF$.
\Statex \textbf{Output}: A strongly admissible labelling $\Lab$ where
	$A \in \inn(\Lab)$, the associted min-max numbering $\MMLab$.
\Statex

\State  run Algorithm \ref{alg-construct}
\State  $\Lab_I \leftarrow \Lab$
\State  $\MMLabI \leftarrow \MMLab$
\State  run Algorithm \ref{alg-prune}
\State  $\Lab \leftarrow \Lab_O$
\State  $\MMLab \leftarrow \MMLabO$
\end{algorithmic}
\end{algorithm}

As an example, consider again the argumentation framework of Figure
\ref{fig-example-AF}. Let $C$ be the main argument. Running Algorithm
\ref{alg-construct} yields a labelling $(\{A,C,D\}, \{B\}, \{E,F,H,H\})$
with associated numbering $\{(A:1), (B:2), (C:3), (D:1)\}$ (as explained in 
Section \ref{subsec-construct}). Feeding this labelling and numbering into 
Algorithm \ref{alg-prune} then yields an output labelling $(\{A,C\}, \{B\}, 
\{D,E,F,G,H\})$ with associated output numbering $\{(A:1), (B:2), (C:3)\}$ 
(as explained in Section \ref{subsec-prune}).

Given the properties of Algorithm \ref{alg-construct} and Algorithm 
\ref{alg-prune}, we can prove that Algorithm \ref{alg-combine} terminates, 
correctly computes a strongly admissible labelling and its associated min-max 
numbering, and runs in polynomial time (more specific, in cubic time).

\begin{theorem} \label{th-a3-terminates}
Let $\AF = (\Arguments, \attack)$ be an argumentation framework 
and $A$ be an argument in the grounded extension of $\AF$. Let both $\AF$
and $A$ be given as input to Algorithm \ref{alg-combine}. It holds that the 
algorithm terminates.
\end{theorem}

\begin{proof}
This follows from Theorem \ref{th-a1-terminates} 
and Theorem \ref{th-a2-terminates}.
\end{proof}

\begin{theorem} \label{th-a3-correct}
Let $\AF = (\Arguments, \attack)$ be an argumentation framework and 
let $A$ be an argument in the grounded extension of $\AF$. Let both $\AF$ and 
$A$ be given as input to Algorithm \ref{alg-combine}. Let $\Lab$ and $\MMLab$ 
be the output of the algorithm. It holds that $\Lab$ is a strongly admissible 
labelling that labels $A$ $\inn$ and has $\MMLab$ as its min-max numbering.
\end{theorem}

\begin{proof}
This follows from Theorem \ref{th-a1-correct}, Theorem \ref{th-LabO-admissible},
Theorem \ref{th-MMLabO-correct} and Theorem \ref{th-LabO-strongly-admissible}.
\end{proof}

\begin{theorem} \label{th-a3-polynomial}
Let $\AF = (\Arguments, \attack)$ be an argumentation framework and 
let $A$ be an argument in the grounded extension of $\AF$. Let both $\AF$ and 
$A$ be given as input to Algorithm \ref{alg-combine}. Let $\Lab$ and $\MMLab$ 
be the output of the algorithm. It holds that Algorithm \ref{alg-combine}
computes $\Lab$ and $\MMLab$ in $O(n^3)$ time.
\end{theorem}

\begin{proof}
This follows from Theorem \ref{th-a1-polynomial} and 
Theorem \ref{th-a2-polynomial}.
\end{proof}

\begin{theorem} \label{th-Lab3-LabI}
Let $\AF = (\Arguments, \attack)$ be an argumentation framework, $A$ be an 
argument in the grounded extension of $\AF$.Let $\AF$ and $A$ be given as input to 
Algorithm \ref{alg-construct} and Algorithm \ref{alg-combine}. Let $\Lab_I$ and $\MMLabI$ be the output of
Algorithm \ref{alg-construct} and let $\Lab_3$ and $\MMLabThree$ be the output of Algorithm \ref{alg-combine}. It holds that $\Lab_3 \sqsubseteq \Lab_I$
\end{theorem}

\begin{proof}
This follows from Theorem \ref{th-LabO-LabI}, together with Theorem \ref{th-a1-correct} and the way Algorithm \ref{alg-combine} is defined (by successively applying Algorithm \ref{alg-construct} and Algorithm \ref{alg-prune})
\end{proof}

%% file: sec04-empirical-results.tex
\section{Empirical Results} \label{sec-empirical-results}

Now that the correctness of our algorithms has been proved and their
computational complexity has been stated, the next step is to empirically
evaluate their performance. For this, we compare both their runtime and
output with that of other computational approaches.

\subsection{Minimality} \label{subsec-minimality}

Although Algorithm 3 aims to find a relatively small strongly admissible
labelling, it is not guaranteed to find an absolute smallest. This is because
the problem of finding the absolute smallest admissible labelling is
coNP-complete, whereas Algorithm 3 is polynomial
(Theorem \ref{th-a3-polynomial}). In essence, we have given up absolute 
minimality in order to achieve tractability. The question, therefore, is how 
much we had to compromise on minimality. That is, how does the outcome of 
Algorithm 3 compare with what would have been an absolute minimal outcome? In order to make the comparison, we will apply the ASPARTIX ASP encodings of
\cite{DW20} to determine the absolute minimal strongly admissible labelling.

Apart from comparing the strongly admissible labelling yielded by our algorithm 
with an absolute \emph{minimal} strongly admissible labelling, we will also
compare it with the absolute \emph{maximal} strongly admissible labelling.
That is, we will compare it with the grounded labelling. The reason for doing
so is that the grounded semantics algorithms (e.g. \cite{MC09,NAD21}) are to
the best of our knowledge currently the only polynomial algorithms for computing a strongly admissible labelling (in particular, for the \emph{maximal} strongly admissible labelling) that have been stated in the literature. As Algorithm 3 is also polynomial
(Theorem \ref{th-a3-polynomial}) this raises the question of how much 
improvement is made regarding minimality.

For queries, we considered the argumentation frameworks in the benchmark sets of ICCMA’17 and ICCMA’19. For each of the argumentation frameworks we generated a query argument that is within the grounded extension (provided the grounded extension is not empty). We used the queried argument when one was provided by the competition (for instance, when considering the benchmark examples of the \emph{Admbuster} class, we took 'a' to be the queried argument as this was suggested by the authors of this class). After considering 514 argumentation frameworks, we found that 277 argumentation frameworks yielded a grounded extension that is not empty (meaning they could used for current purposes). 

We conducted our experiments on a MacBook Pro 2020 with 8GB of memory and an
Intel Core i5 processor. To run the ASPARTIX system we used clingo v5.5.1. We set a timeout limit of 1000 seconds and a memory limit of 8GB per query.

For each of the selected benchmark examples, we have assessed the following:
\begin{enumerate}
  \item the size of the grounded labelling 
	(determined using the modified version of Algorithm 1 
	as described in Lemma \ref{lemma-a1-modified-grounded})
  \item the size of the strongly admissible labelling yielded by Algorithm 1
  \item the size of the strongly admissible labelling yielded by Algorithm 3
  \item the size of the absolute minimal strongly admissible labelling
	(yielded by the approach of \cite{DW20})
\end{enumerate}

We start our analysis with comparing the output of Algorithm \ref{alg-construct} and Algorithm \ref{alg-combine} with the grounded labelling regarding the size of the respective labellings. We found that the size of the strongly admissible labelling yielded by Algorithm \ref{alg-construct} tends to be smaller than the size of the grounded labelling. More specifically, the strongly admissible labelling yielded by Algorithm \ref{alg-construct} is smaller than the size of the grounded labelling in 63\% of the 277 examples we tested for. In the remaining 37\% of the examples, their sizes are the same. 

Figure \ref{fig-MAlg1GL} provides a more detailed overview of our findings, in the form of a bar graph. The rightmost bar represents the 37\% of the cases where the output of Algorithm \ref{alg-construct} has the same size as the grounded labelling (that is, where the size of the output of Algorithm \ref{alg-construct} is 100\% of the size of the grounded labelling). The bars on the left of this are for the cases where the size of the output of Algorithm \ref{alg-construct} is less than the size of the grounded labelling. For instance, it was found that in 10\% of the examples, the size of output of Algorithm \ref{alg-construct} is 80\% to 89\% of the size of the grounded labelling. On average, we found that the size of the output of Algorithm \ref{alg-construct} is 76\% of the size of the grounded labelling.

\begin{figure}[thb]
\includegraphics[width=15cm]{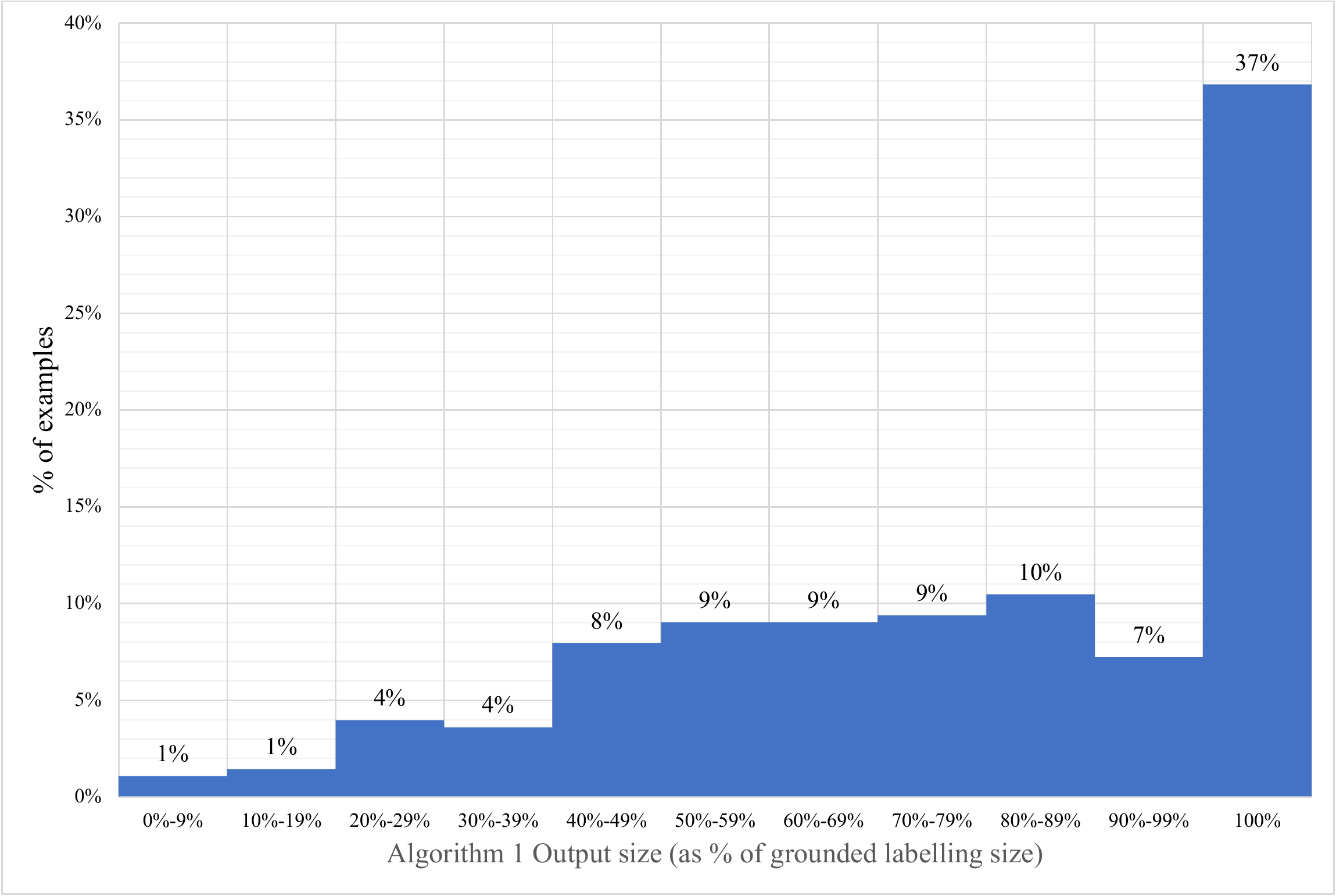}
\caption{The size of output of Algorithm 1 (as a percentage of the grounded labelling).
\label{fig-MAlg1GL}}
\centering
\end{figure}

As for Algorithm \ref{alg-combine}, we found an even bigger improvement in the size of it's output labelling compared to the grounded labelling. More specifically, the size of the strongly admissible labelling yielded by Algorithm \ref{alg-combine} is smaller than the grounded labelling in 88\% of the 277 examples we tested for. Figure \ref{figAlg3GL} provides a more detailed overview of our findings in a similar way as we previously did for Algorithm \ref{alg-construct}. On average, we found that the output of Algorithm \ref{alg-combine} has a size that is 25\% of the size of the grounded labelling. 

\begin{figure}[thb]
\includegraphics[width=15cm]{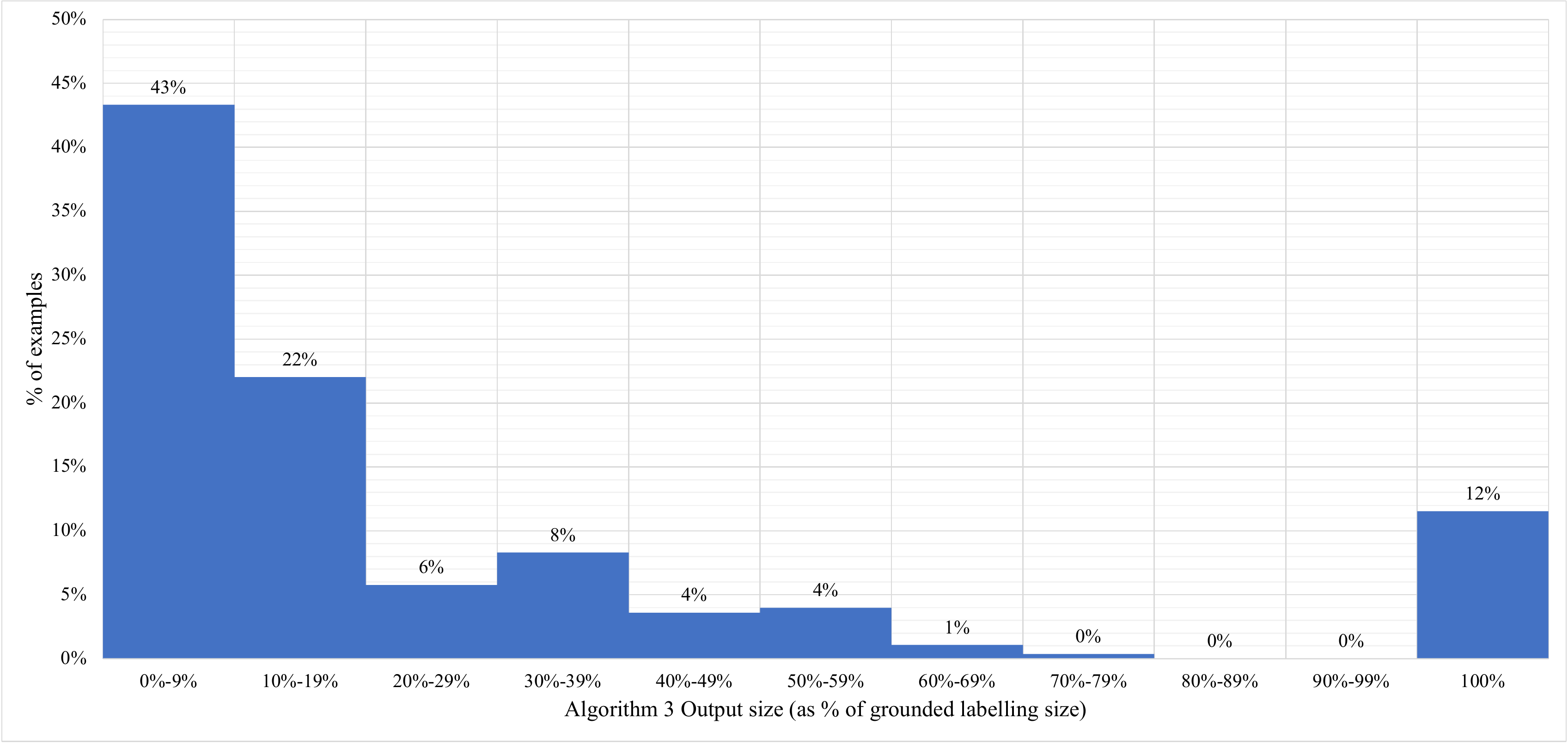}
\caption{The size of output of Algorithm 3 (as a percentage of the grounded labelling).
\label{figAlg3GL}}
\centering
\end{figure}

Apart from comparing Algorithm \ref{alg-construct} and Algorithm \ref{alg-combine} with the grounded labelling, it can also be insightful to compare the two algorithms with each other. In figure \ref{fig-MA1A3-GL}, each dot represents one of the 277 examples.\footnote{Please be aware that some of the dots overlap each other.} The horizontal axis represents the size of the output of Algorithm \ref{alg-construct}, as a percentage of the size of the grounded labelling. The vertical axis represents the size of the output of Algorithm \ref{alg-combine} as a percentage of the size of the grounded labelling. For easy reference, we have included a dashed line indicating the situation where the output of Algorithm \ref{alg-construct} has the same size as the output of Algorithm \ref{alg-combine}. Any dots below the dashed line represent the cases where Algorithm \ref{alg-combine} outperforms Algorithm \ref{alg-construct}, in that it yields a smaller strongly admissible labelling. Any dots above the dashed line represents the cases where Algorithm \ref{alg-combine} under performs Algorithm \ref{alg-construct} in that it yields a bigger strongly admissible labelling. Unsurprisingly, there no such cases as Theorem \ref{th-Lab3-LabI} states that the output of Algorithm \ref{alg-combine} cannot be bigger than the output of Algorithm \ref{alg-construct}. 

We found that for 95\% of the examples, Algorithm \ref{alg-combine} produces a smaller labelling than Algorithm \ref{alg-construct}. Moreover, we found that on average, the output of Algorithm 3 is 32\% smaller than output of Algorithm \ref{alg-construct}.

\begin{figure}[thb]
\includegraphics[width=15cm]{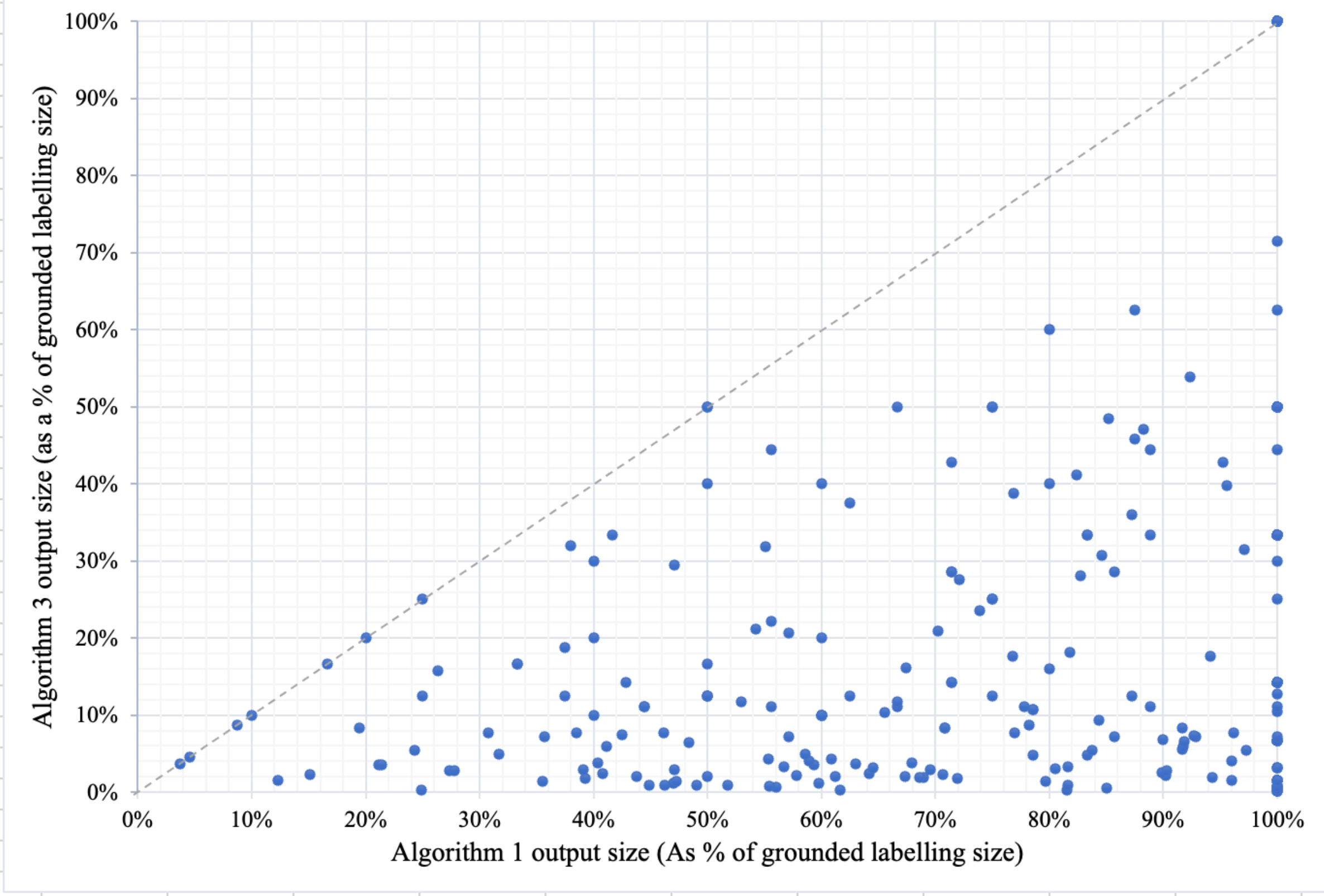}
\caption{The size of output of Algorithm 1 compared to the output Algorithm 3 (as a percentage of the size of the grounded labelling).
\label{fig-MA1A3-GL}}
\centering
\end{figure}

The next question is how the output of our best performing algorithm (Algorithm \ref{alg-combine}) compares with what would have been the ideal output. That is, we compare the size of the output of Algorithm \ref{alg-combine} with the size of an minimal strongly admissible labelling for the main argument in question, as computed using the ASPARTIX encodings of \cite{DW20}. The results are shown in Figure \ref{fig-A3ASP-GL}.

\begin{figure}[thb]
\includegraphics[width=15cm]{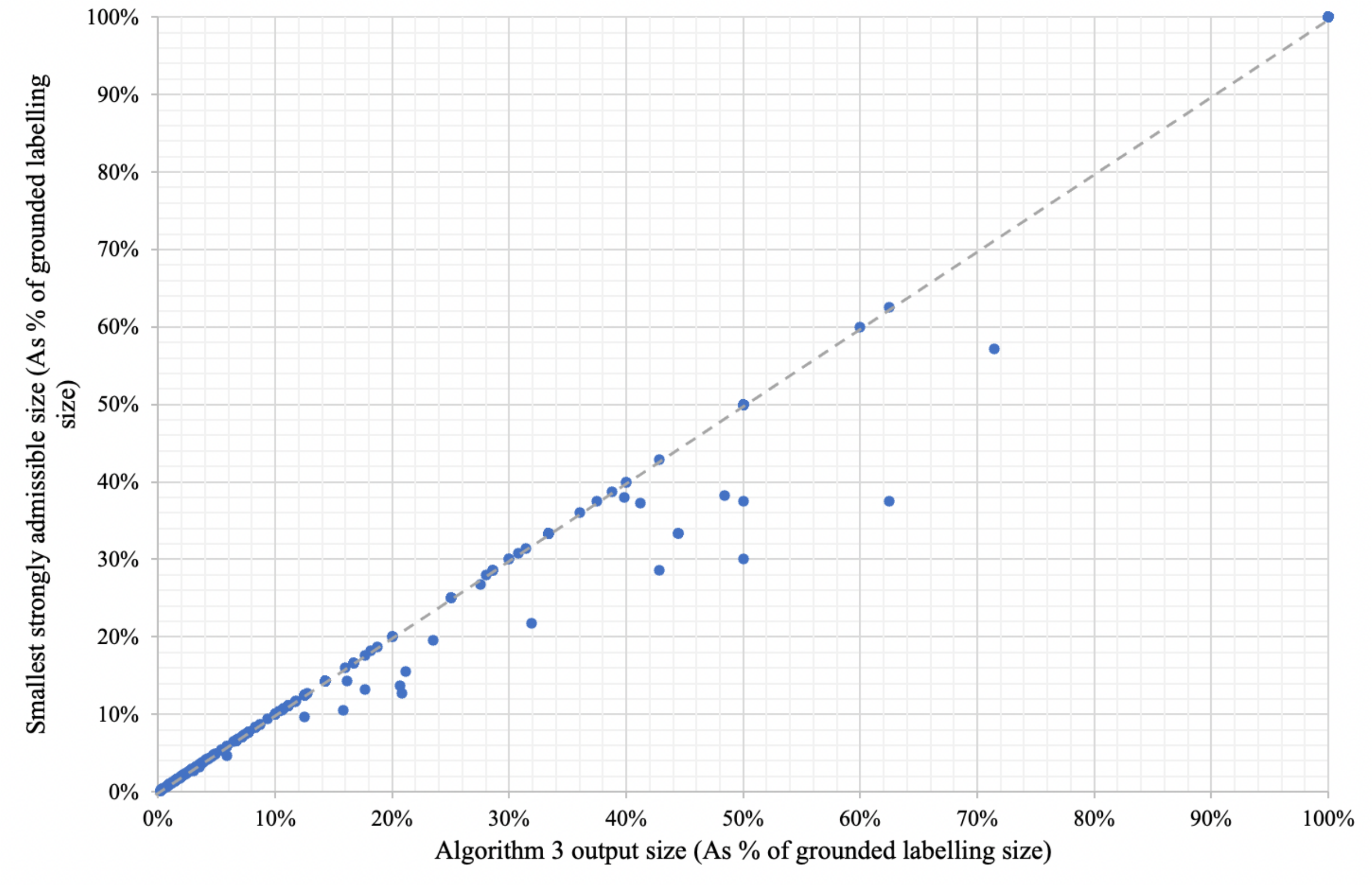}
\caption{The size of output of Algorithm 3 compared to the smallest strongly admissible labelling (as a percentage of the size of the grounded labelling).
\label{fig-A3ASP-GL}}
\centering
\end{figure}

We found that in 91\% of the 277 examples, the output of Algorithm \ref{alg-combine} is of the same size as the smallest strongly admissible labelling for the output of the main argument in question. For the other 9\% of the examples, the output of Algorithm \ref{alg-combine} has a bigger size. On average, we found that the output of Algorithm \ref{alg-combine} is 3\% bigger than the smallest strongly admissible labelling for the main argument in question. Figure \ref{fig-A3ASP-SL}, provides a more detailed overview of how much bigger the output of Algorithm \ref{alg-combine} is compared to the smallest strongly admissible labelling for the main argument in question. 

\begin{figure}[thb]
\includegraphics[width=15cm]{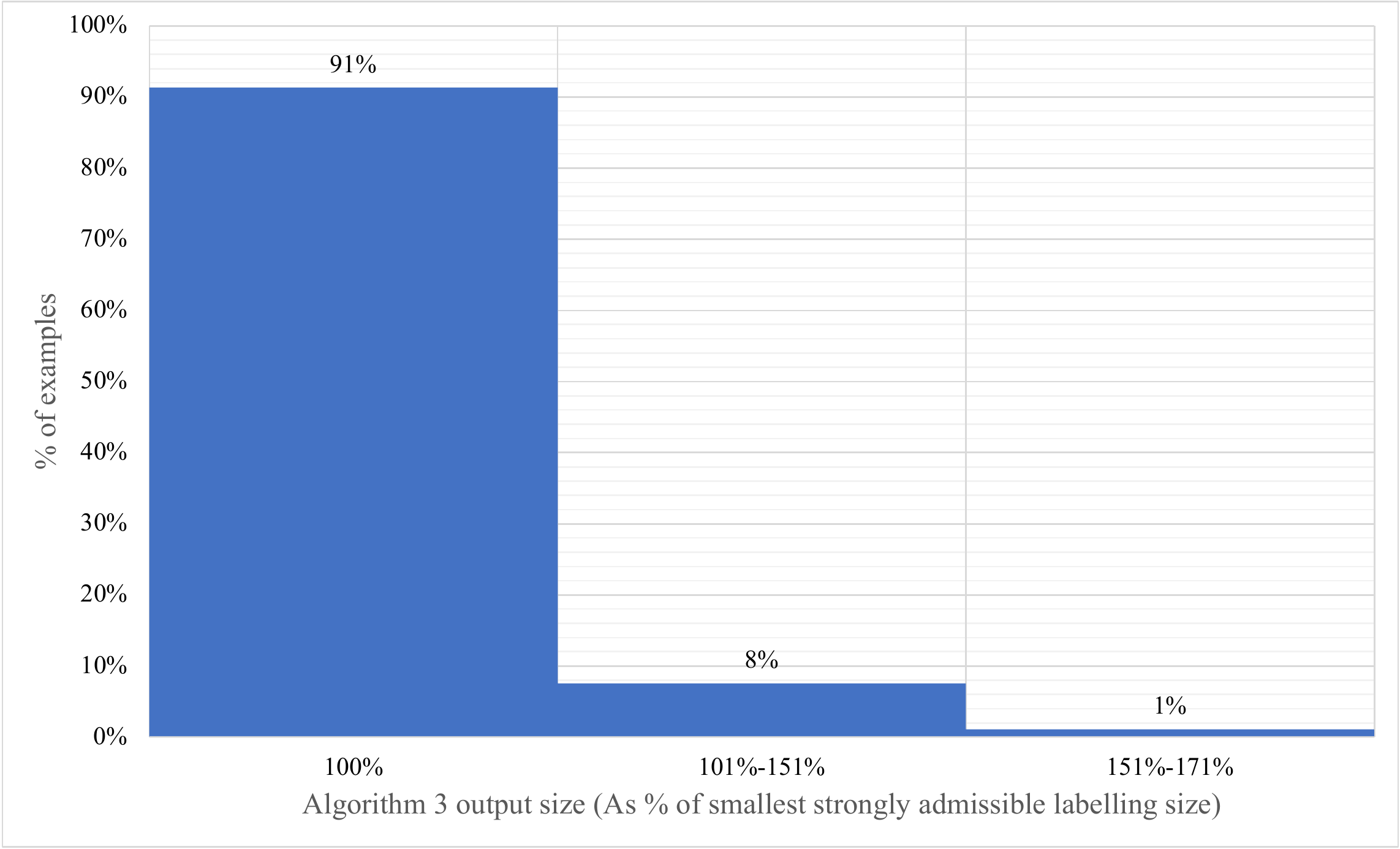}
\caption{The size of output of Algorithm 3 compared to the smaller strongly admissible labelling (as a percentage of the size of the grounded labelling).
\label{fig-A3ASP-SL}}
\centering
\end{figure}

\subsection{Runtime} \label{subsec-runtime}

The next thing to study is how the runtime of our algorithms compares with
the runtime of some of the existing computational approaches. In particular,
we compare the runtime of Algorithm 1 and Algorithm 3 with the runtime of
the ASPARTIX-based approach of \cite{DW20}.

We first compare the runtime of Algorithm \ref{alg-combine} to the runtime of the modified version of Algorithm 3 of \cite{NAD21} for computing the grounded labelling. It turns out that the runtimes of these algorithms are very similar. On average, Algorithm 3 of \cite{NAD21} took 0.02(3\%) seconds more than Algorithm \ref{alg-combine} to solve the test instances.
These runtime results of Algorithm \ref{alg-combine} and Algorithm 3 of \cite{NAD21} are illustrated within Figure \ref{fig-RT-A3-GL}.  

\begin{figure}[thb]
\includegraphics[width=15cm]{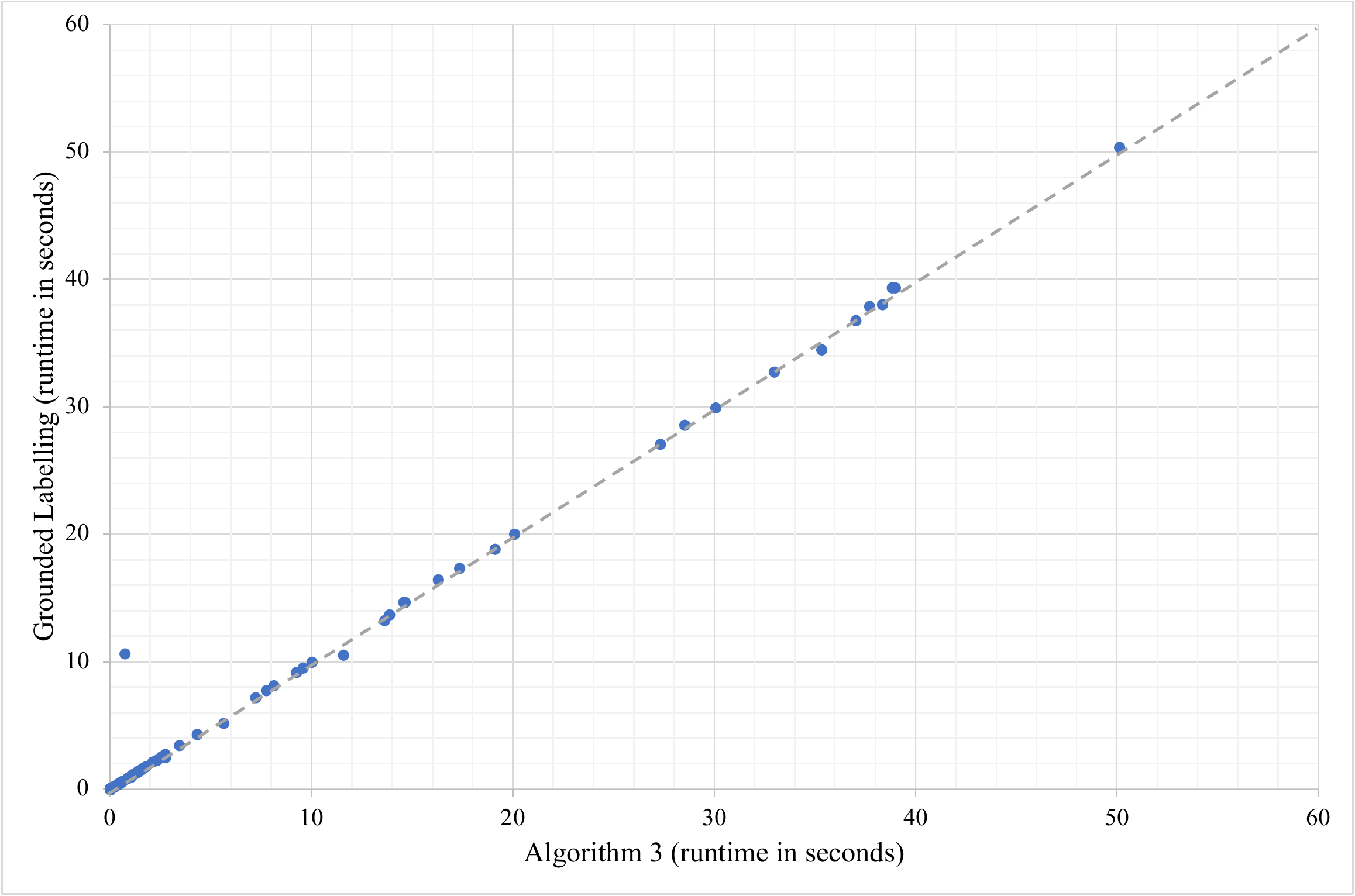}
\caption{The runtime of computing Algorithm 3 compared to the runtime of computing the grounded labelling).
\label{fig-RT-A3-GL}}
\centering
\end{figure}

The next question is how does the runtime of computing Algorithm \ref{alg-combine} compare to the runtime of the ASPARTIX encoding for minimal strongly admissibility. It was observed that the runtime of ASPARTIX encoding is significantly longer than the runtime of Algorith \ref{alg-combine}. A detailed overview of the difference in runtimes of Algorithm \ref{alg-combine} and the ASPARTIX encoding on minimal strong admissibility is shown in Figure \ref{fig-RT-A3-ASP}. On average, the ASPARTIX framework took 12.5 seconds (907\%) more than Algorithm \ref{alg-combine} to solve the test instances.

\begin{figure}[thb]
\includegraphics[width=15cm]{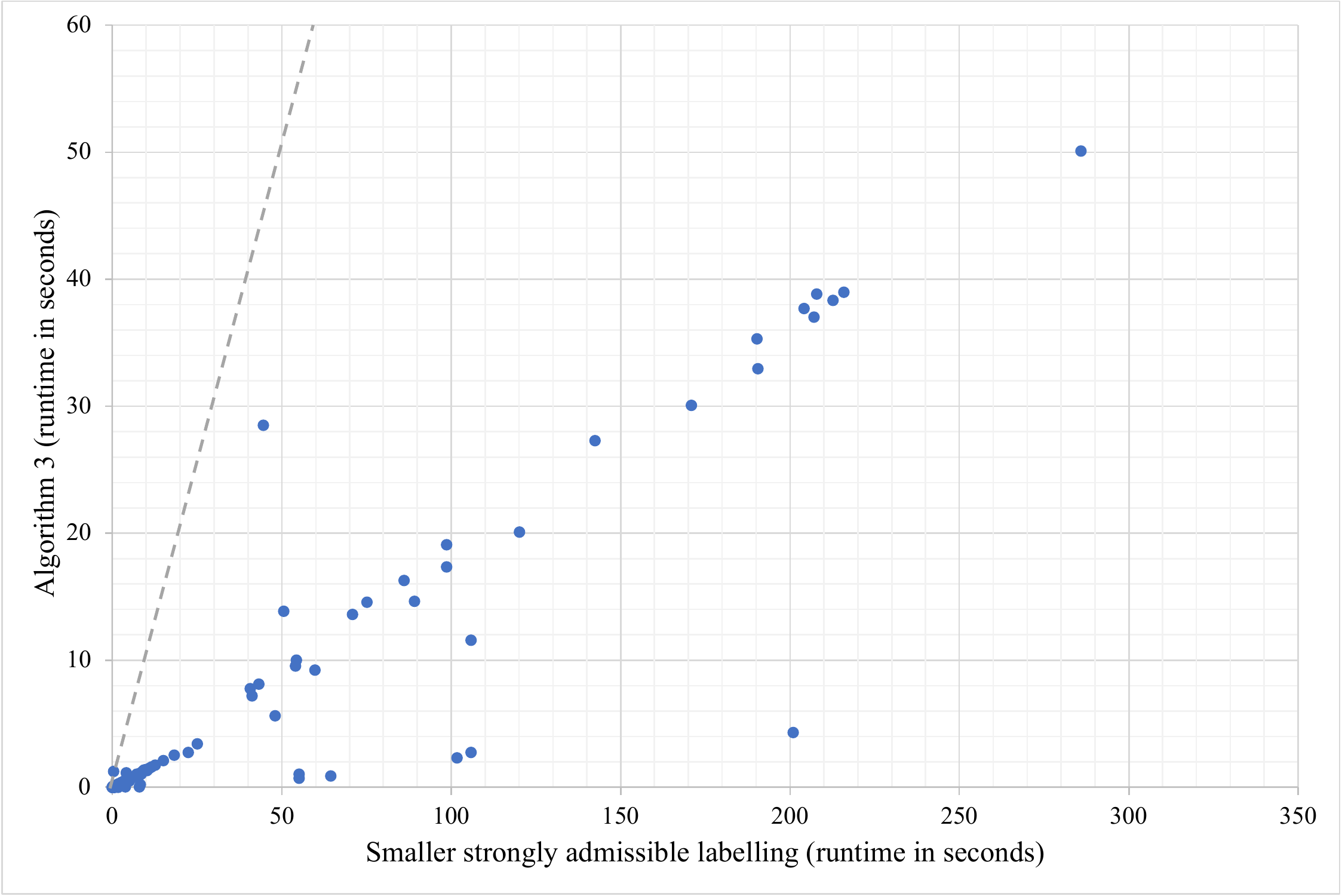}
\caption{The runtime of computing Algorithm 3 compared to the runtime of computing the ASPARTIX encoding on minimal strong admissibility.
\label{fig-RT-A3-ASP}}
\centering
\end{figure}

%% file: sec05-discussion.tex
\section{Discussion} \label{sec-discussion}
In the current paper, we provided two algorithms (Algorithm \ref{alg-construct} and Algorithm \ref{alg-combine}) for computing a relatively small strongly admissible labelling for an argument that is in the grounded extension. We proved that both algorithms are correct in the sense that each of them returns a strongly admissible labelling (with associated min-max numbering) that labels the main argument in question $\inn$ (Theorem \ref{th-a1-correct} and \ref{th-a3-correct}). Moreover, each algorithm runs in polynomial (cubic) time (Theorem \ref{th-a1-polynomial} and Theorem \ref{th-a3-polynomial}). It was also shown that the strongly admissible labelling yielded by Algorithm \ref{alg-combine} is smaller than or equal to the strongly admissible labelling yielded by Algorithm \ref{alg-construct} (Theorem \ref{th-LabO-LabI}). 

The next question we examined is how small the output of Algorithm \ref{alg-construct} and Algorithm \ref{alg-combine} is compared to the smallest strongly admissible for the main argument in question. Unfortunately, previous findings make it difficult to provide formal theoretical results on this. This is because the k-approximation problem for strong admissibility is NP-hard, meaning that a polynomial algorithm (such as Algorithm \ref{alg-construct} and Algorithm \ref{alg-combine}) cannot provide any guarantees of yielding a result within a fixed parameter $k$ from the size of the absolute smallest strongly admissible labelling for the main argument in question. 

Hence, instead of developing theoretical results, we decided to approach the issue of minimality in an empirical way, using a number of experiments. These experiments were based on the benchmark examples that were submitted to ICCMA'17 and ICCMA'19. We compared the output of Algorithm \ref{alg-construct} and Algorithm \ref{alg-combine} with both the biggest and the smallest strongly admissible set for the main argument in question (the biggest was computed using Algorithm 3 described in \cite{NAD21} and the smallest was computed using the ASPARTIX based approach on computing minimal strong admissibility in \cite{DW20}). Overall, we found that Algorithm \ref{alg-combine} yields results that are only marginally bigger than the smallest strongly admissible labelling, with a run-time that is a fraction of the time that would be required to find this smallest strongly admissible labelling. The outputs of both or algorithms return a strongly admissible labelling that is significantly smaller than the biggest strongly admissible labelling (the grounded labelling), with the output of Algorithm \ref{alg-combine} on average being only 25\% of the output of the biggest strongly admissible labelling. 

The research of the current paper fits into our long-term research agenda of using argumentation theory to provide explainable formal inference. In our view, it is not enough for a knowledge-based system to simply provide an answer regarding what to do or what to believe. There should also be a way for this answer to be explained. One way of doing so is by means of (formal) discussion. Here, the idea is that the knowledge-based system should provide the argument that is at the basis of its advice. The user is then allowed to raise objections (counterarguments) which the system then replies to (using counter-counter-arguments), etc. In general, we would like such a discussion to be (1) sound and complete for the underlying argumentation semantics, (2) not be unnecessarily long, and (3) be close enough to human discussion to be perceived as natural and convincing
 
As for point (1), sound and complete discussion games have been identified for grounded, preferred, stable and ideal semantics \cite{Cam15a}. As for point (2), this is what we studied in the current paper, as well as in \cite{Cam14a,CD19a}. As for point (3), this is something that we are aiming to report on in future work.

For future research, it is possible to conduct a similar sort of analysis (as in this paper) on minimal admissible labellings. It was reported obtaining an absolute minimal admissible labelling for a main argument in question is also of coNP-complete complexity \cite{Cam14a,CD19a} therefore, it would be interesting to look into developing an algorithm that generates a small admissible labelling in polynomial time complexity. Similarly, it would also be interesting to look at the complexity and empirical results on generating minimal ideal sets. 

%% file: sec06-epilogue.tex
\section{Epilogue} \label{sec-epilogue}

Although the main topic of the current paper is how to construct a relatively small strongly admissible labelling(for a particular argument) in a time-efficient way, our results also allow us to provide an analysis of two adjacent questions: what is the additional cost of computing the min-max numbering compared to only computing the strongly admissible labelling itself and what is the fastest approach for computing \emph{any} strongly admissible labelling (for a particular argument) if the size of the labelling does not matter. In the following two sections, we study the questions in more detail. 

\subsection{The additional costs of computing the min-max numbering} \label{subsec-addtionalcost-minmax}

As we mentioned earlier, our approach (in particular Algorithm \ref{alg-construct}) is based on the work of \cite{NAD21}. However, where the works of \cite{NAD21} only computes a strongly admissible labelling (the biggest strongly admissible labelling, to be precise) our approach additionally computes the associated min-max numbering. This raises the question of what is the additional runtime needed to compute this min-max numbering. 

In order to a like-for-like comparison, we compare the runtime of Algorithm 3 of \cite{NAD21} with the runtime of our own Algorithm for computing the grounded labelling and it's associated min-max numbering, as described by Lemma \ref{lemma-a1-modified-grounded}. Each of the two algorithms was run on 277 examples of the earlier mentioned testset. The results are provided in Figure \ref{fig-RT-Leamm9-Nofal}. On average, the runtime of the algorithm of Lemma \ref{lemma-a1-modified-grounded} is 0.0004\% longer than the runtime of Algorithm 3 of \cite{NAD21}.

\begin{figure}[thb]
\includegraphics[width=15cm]{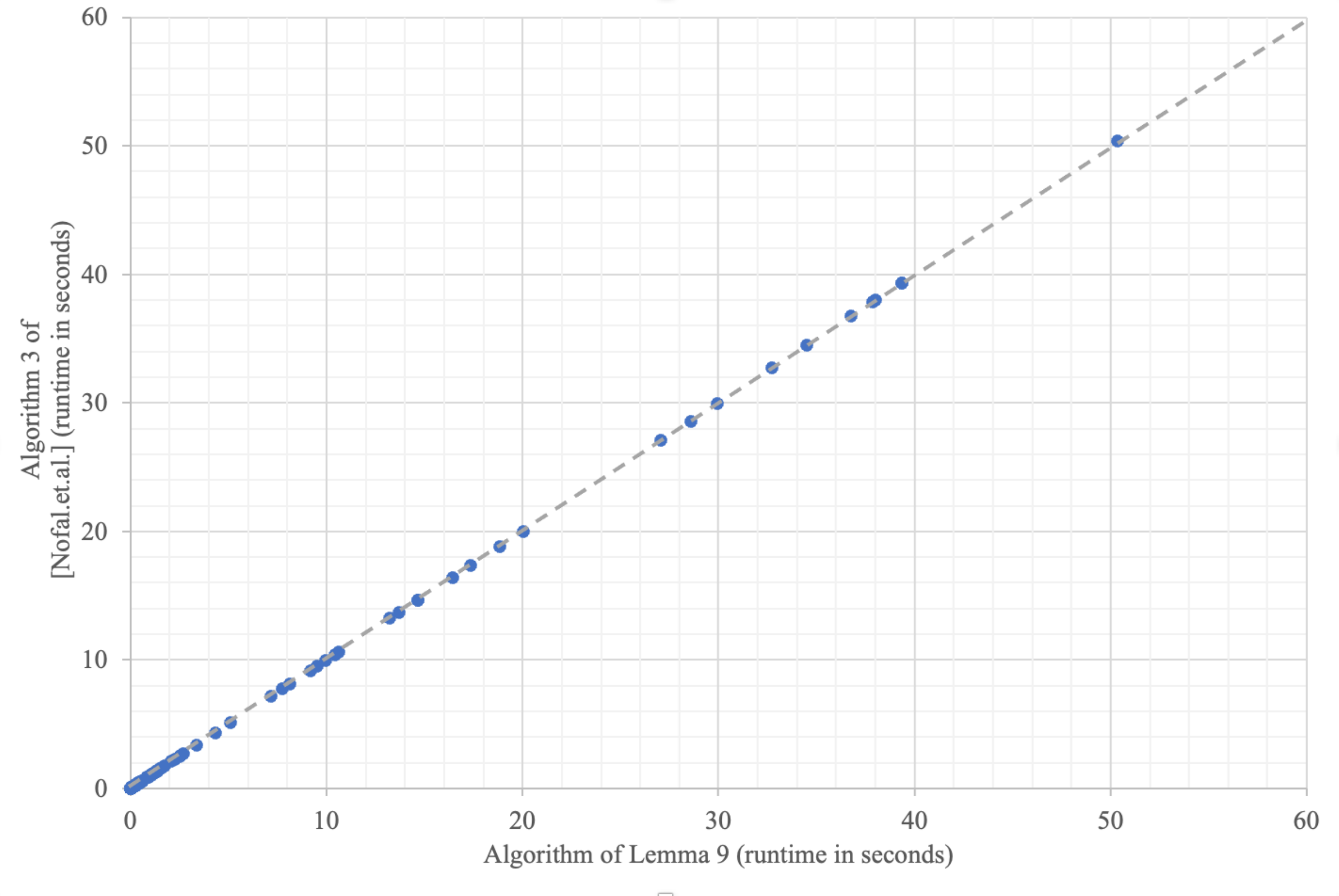}
\caption{The runtime of computing the Algorthim of Lemma \ref{lemma-a1-modified-grounded} compared to the runtime of computing Algorithm 3 of \cite{NAD21}).
\label{fig-RT-Leamm9-Nofal}}
\centering
\end{figure}

It is also possible to do a like-for-like comparison w.r.t Algorithm \ref{alg-construct}, by comparing the runtime of the algorithm itself with the runtime of the algorithm after commenting out lines 15, 26 and 32 (which are used to compute the min-max numbering). The results are provided in Figure \ref{fig-RT-Alg1CmmtedOutAlg1}. On average, we found that the runtime of Algorithm \ref{alg-construct} is 3\%(0.067 seconds) longer than the runtime of Algorithm \ref{alg-construct} with lines 15, 26 and 32 commented out. 

\begin{figure}[thb]
\includegraphics[width=15cm]{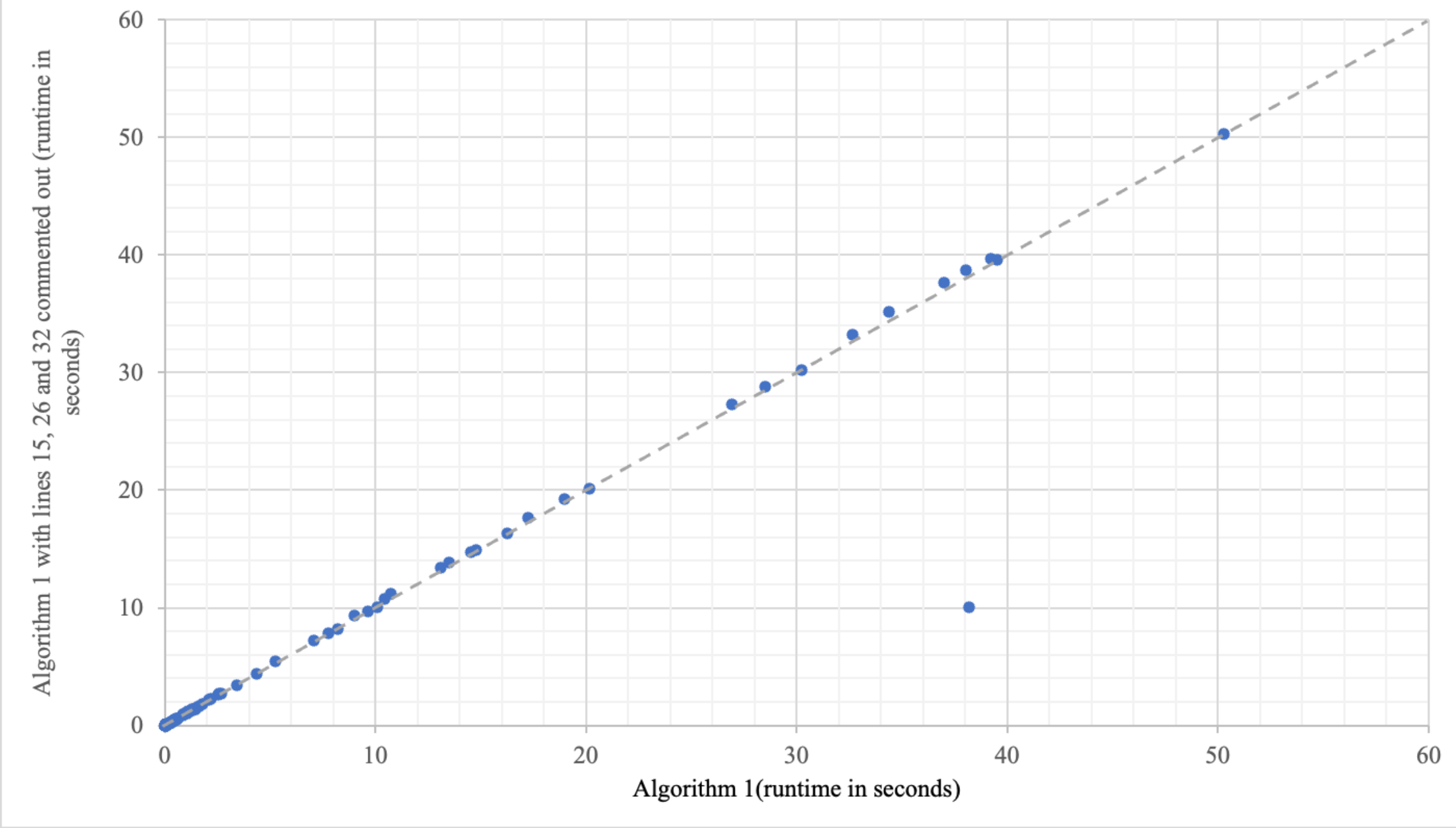}
\caption{The runtime of computing the Algorthim \ref{alg-construct} with the runtime of computing Algorithm \ref{alg-construct} with lines 15, 26 and 32 commented out.
\label{fig-RT-Alg1CmmtedOutAlg1}}
\centering
\end{figure}

Overall, we observe that the additional runtime for computing the min-max numbering is only marginally higher than the runtime for computing only the strongly admissible labelling itself. 
As an aside, the reader might wonder why we did not carry out a similar like-for-like comparison in the context of Algorithm \ref{alg-combine}. That is, why did we not compare the runtime of Algorithm \ref{alg-combine} with the runtime of a modified version of Algorithm \ref{alg-combine} in which all computation of the min-max numbering has been commented out? The reason for not doing so is that Algorithm \ref{alg-combine} contains Algorithm \ref{alg-prune} whose correctness critically depends on the presence of a min-max numbering. To illustrate this, consider the argumentation framework of Figure \ref{fig-square-af}. Suppose $E$ is the main argument in question. Algorithm \ref{alg-construct} in its unmodified form will yield the strongly admissible labelling $(\{A,C,E\}, \{B,D\}, \emptyset$) and it's associated min-max numbering $\{(A:1),(B:2),(C:3),(D:4),(E,5)\}$. 
Algorithm \ref{alg-prune}, in it's unmodified form will use this min-max numbering once it arrives at argument $B$ for selecting a \emph{minimally numbered} in-labelled attacker of $B$ (which is $A$). However, without the numbering, Algorithm \ref{alg-prune} would not know whether to choose $A$ or $E$ as the attacker of $B$. In the absence of a min-max numbering, the algorithm could decide that $B$ already has an in-labelled attacker ($E$), rather than adding the \emph{minimally numbered} in-labelled attacker ($A$), resulting in the incorrect strongly admissible labelling of $(\{C,E\},\{B,D\},\{A\})$.
Hence, we cannot compare Algorithm \ref{alg-combine} with a modified version of \ref{alg-combine} that does not dedicate any resources for computing the min-max numbering, as the latter algorithms would be guaranteed to be correct. 

\begin{figure}[thb]
\centering
\includegraphics[scale=1.0]{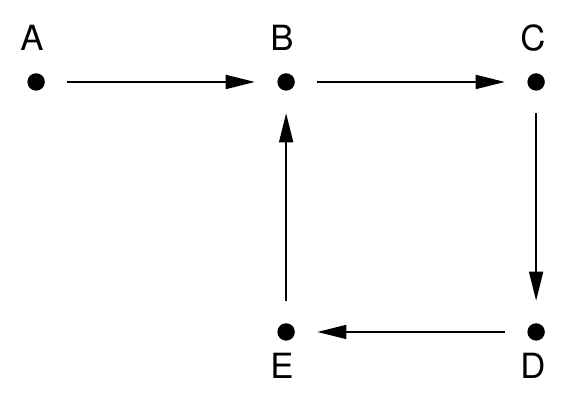}
\caption{Example argumentation framework \label{fig-square-af}}
\end{figure}

\subsection{Computing an \emph{arbitrary} strongly admissible labelling} \label{subsec-arbitrary-strong-admissible-labelling}

So far, we have focused our attention on the problem of finding a \emph{small} strongly admissible labelling (for a particular argument) in a time efficient way. We now examine, the question of how to find an \emph{arbitrary} strongly admissible labelling (for a particular argument) in a time efficient way. That is, we are interested in a fast way of constructing a strongly admissible labelling (that labels the argument in question $\inn$) without caring how big or small the labelling is.\footnote{It can be mentioned that ICCMA'17 and ICCMA'19 describe the somewhat similar task of finding an arbitrary extension for a particular semantics.}
In section \ref{subsec-runtime}, we compared the runtime of Algorithm \ref{alg-construct} and \ref{alg-combine} with the runtime of the ASPARTIX-based approach of \cite{DW20}. The complete our analysis, in the current section we will also compare the runtime of the algorithm of Lemma \ref{lemma-a1-modified-grounded} (for computing the grounded labelling) with the run-time of the ASPARTIX-based approach of \cite{DW20}. The results are provided in the Figure \ref{fig-RT-ASPARTIX-Lemma9}. On average, the runtime of the Algorithm of Lemma \ref{lemma-a1-modified-grounded} is 16\% of the runtime of the ASPARTIX-based approach of \cite{DW20}.

\begin{figure}[thb]
\includegraphics[width=15cm]{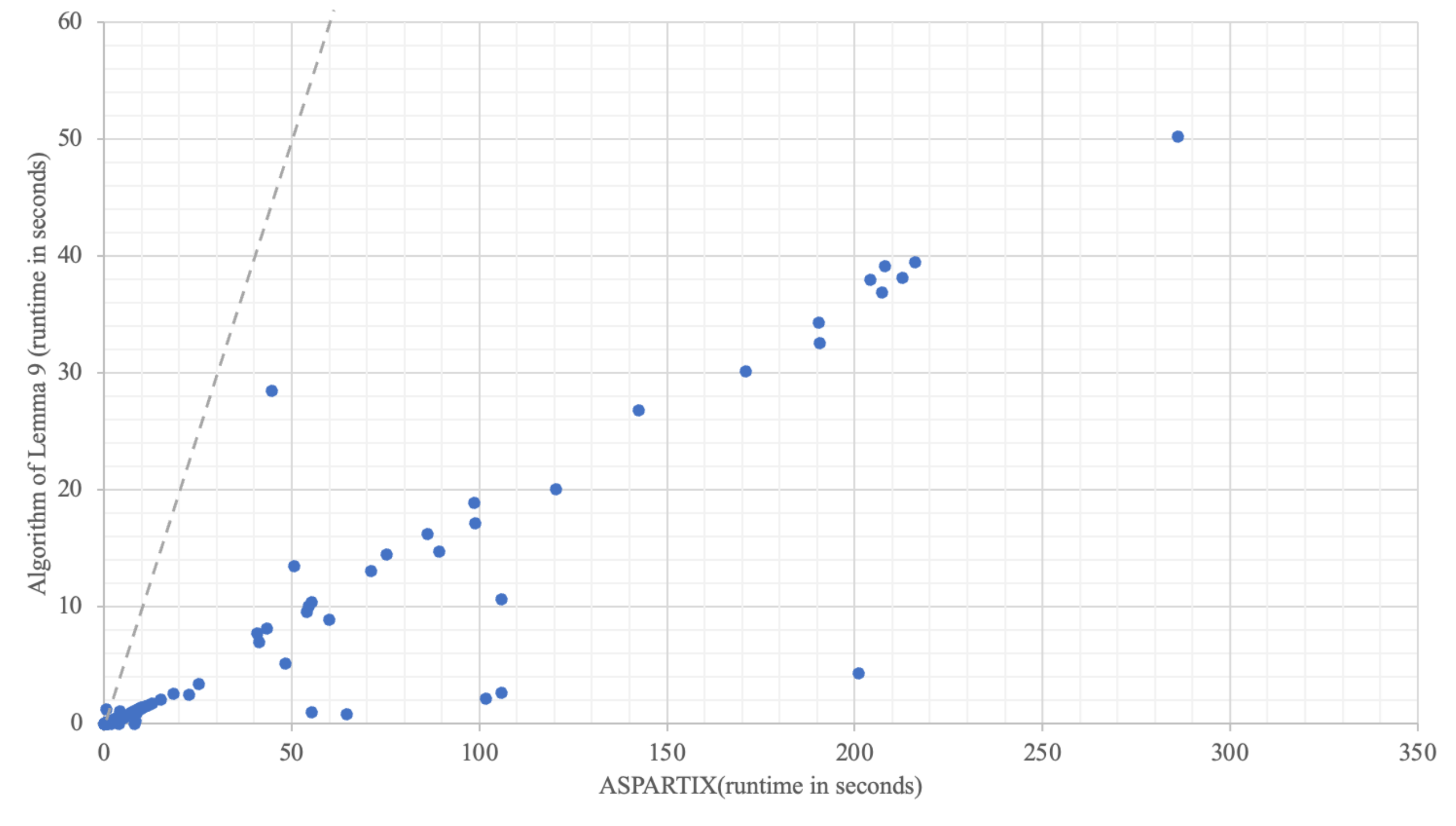}
\caption{The runtime of computing the Algorithm of Lemma \ref{lemma-a1-modified-grounded} with the ASPARTIX-based approach of \cite{DW20}.
\label{fig-RT-ASPARTIX-Lemma9}}
\centering
\end{figure}

To summarise our results, we found that compared to the ASPARTIX-based approach of \cite{DW20}
\begin{enumerate}
  \item the runtime of Algorithm \ref{alg-construct} is on average 16\%,
  \item the runtime of Algorithm \ref{alg-combine} is on average 16\%, and
  \item the runtime of Algorithm of Lemma \ref{lemma-a1-modified-grounded} is on average 16\%. 
\end{enumerate}

Hence, when the aim is to find an arbitrary strongly admissible labelling for a particular argument, our findings confirm the expectation that of above mentioned three algorithms, Algorithm \ref{alg-construct} is the most time-efficient approach.\footnote{This is what one would expect to find as Algorithm \ref{alg-construct} is part of Algorithm \ref{alg-combine}, and unlike the Algorithm of Lemma \ref{lemma-a1-modified-grounded}, Algorithm \ref{alg-construct} terminates when encountering the main argument in question.}

%% file: techreport_str_adm_alg.bbl
\begin{thebibliography}{10}

\bibitem{BG07a}
P.~Baroni and M.~Giacomin.
\newblock On principle-based evaluation of extension-based argumentation
  semantics.
\newblock {\em Artificial Intelligence}, 171(10-15):675--700, 2007.

\bibitem{Cam06d}
M.W.A. Caminada.
\newblock On the issue of reinstatement in argumentation.
\newblock In M.~Fischer, W.~van~der Hoek, B.~Konev, and A.~Lisitsa, editors,
  {\em Logics in Artificial Intelligence; 10th European Conference, JELIA
  2006}, pages 111--123. Springer, 2006.
\newblock LNAI 4160.

\bibitem{Cam06a}
M.W.A. Caminada.
\newblock On the issue of reinstatement in argumentation.
\newblock Technical Report UU-CS-2006-023, Institute of Information and
  Computing Sciences, Utrecht University, 2006.

\bibitem{Cam14a}
M.W.A. Caminada.
\newblock Strong admissibility revisited.
\newblock In S.~Parsons, N.~Oren, C.~Reed, and F.~Cerutti, editors, {\em
  Computational Models of Argument; Proceedings of COMMA 2014}, pages 197--208.
  IOS Press, 2014.

\bibitem{Cam15a}
M.W.A. Caminada.
\newblock A discussion game for grounded semantics.
\newblock In E.~Black, S.~Modgil, and N.~Oren, editors, {\em Theory and
  Applications of Formal Argumentation (proceedings TAFA 2015)}, pages 59--73.
  Springer, 2015.

\bibitem{BCG17}
M.W.A. Caminada, P.~Baroni, and M.~Giacomin.
\newblock Abstract argumentation frameworks and their semantics.
\newblock In {\em Handbook of Formal Argumentation}, volume~1. College
  Publications, 2018.

\bibitem{CD19a}
M.W.A. Caminada and P.E. Dunne.
\newblock Strong admissibility revised: theory and applications.
\newblock {\em Argument \& Computation}, 10:277--300, 2019.

\bibitem{CD20a}
M.W.A. Caminada and P.E. Dunne.
\newblock Minimal strong admissibility: a complexity analysis.
\newblock In H.~Prakken, S.~Bistarelli, F.~Santini, and C.~Taticchi, editors,
  {\em Proceedings of COMMA 2020}, pages 135--146. IOS Press, 2020.

\bibitem{CG09}
M.W.A. Caminada and D.M. Gabbay.
\newblock A logical account of formal argumentation.
\newblock {\em Studia Logica}, 93(2-3):109--145, 2009.
\newblock Special issue: new ideas in argumentation theory.

\bibitem{CP11}
M.W.A. Caminada and G.~Pigozzi.
\newblock On judgment aggregation in abstract argumentation.
\newblock {\em Autonomous Agents and Multi-Agent Systems}, 22(1):64--102, 2011.

\bibitem{Dun95}
P.M. Dung.
\newblock On the acceptability of arguments and its fundamental role in
  nonmonotonic reasoning, logic programming and $n$-person games.
\newblock {\em Artificial Intelligence}, 77:321--357, 1995.

\bibitem{DW20}
W.~Dvo\v{r}\'{a}k and J.~Wallner.
\newblock Computing strongly admissible sets.
\newblock In H.~Prakken, S.~Bistarelli, F.~Santini, and C.~Taticchi, editors,
  {\em Proceedings of COMMA 2020}, pages 179--190. IOS Press, 2020.

\bibitem{MC09}
S.~Modgil and M.W.A. Caminada.
\newblock Proof theories and algorithms for abstract argumentation frameworks.
\newblock In I.~Rahwan and G.R. Simari, editors, {\em Argumentation in
  Artificial Intelligence}, pages 105--129. Springer, 2009.

\bibitem{NAD21}
S.~Nofal, K.~Atkinson, and P.E. Dunne.
\newblock Computing grounded extensions of abstract argumentation frameworks.
\newblock {\em The Computer Journal}, 64:54--63, 2021.

\end{thebibliography}
